\documentclass[twoside,11pt]{article}

\usepackage{jmlr2e}

\usepackage{amsfonts}
\usepackage{mathtools}
\usepackage{algorithm}
\usepackage{algpseudocode}
\usepackage{etoolbox} 
\usepackage{letltxmacro} 

\usepackage{color}

\newcommand{\Mod}[1]{\ \mathrm{mod}\ #1}
\newtheorem{thm}{Theorem}[section]
\newtheorem{cor}[thm]{Corollary}
\newtheorem{lem}[thm]{Lemma}

\newtheorem{defn}[thm]{Definition}

\DeclareMathOperator*{\argmin}{arg\,min}

\newtoggle{graphics}
\LetLtxMacro{\reallyincludegraphics}{\includegraphics}
\renewcommand{\includegraphics}[2][]{\iftoggle{graphics}{\reallyincludegraphics[#1]{#2}}{}}

\toggletrue{graphics}

\jmlrheading{19}{2019}{1-42}{5/18; Revised 8/18, 11/18}{1/19}{18-281}{Veit Elser, Dan Schmidt and Jonathan Yedidia}

\ShortHeadings{Monotone Learning with Rectified Wire Networks}{Elser, Schmidt and Yedidia}
\firstpageno{1}

\begin{document}

\title{Monotone Learning with Rectified Wire Networks} 

\author{\name Veit Elser \email ve10@cornell.edu \\
       \addr Department of Physics\\
       Cornell University\\
       Ithaca, NY 14853-2501, USA
       \AND
       \name Dan Schmidt \email Daniel.Schmidt@analog.com \\
       \and
        \name Jonathan Yedidia \email Jonathan.Yedidia@analog.com \\
       \addr Analog Devices, Inc.\\
       Boston, MA, USA}

\editor{Samy Bengio}

\date{}

\maketitle

\begin{abstract}
We introduce a new neural network model, together with a tractable and monotone online learning algorithm. Our model describes feed-forward networks for classification, with one output node for each class. The only nonlinear operation is rectification using a ReLU function with a bias. However, there is a rectifier on every edge rather than at the nodes of the network. There are also weights, but these are positive, static, and associated with the nodes. Our \textit{rectified wire networks} are able to represent arbitrary Boolean functions.
Only the bias parameters, on the edges of the network, are learned.
Another departure in our approach, from standard neural networks, is that the loss function is replaced by a constraint. This constraint is simply that the value of the output node associated with the correct class should be zero. Our model has the property that the exact norm-minimizing parameter update, required to correctly classify a training item, is the solution to a quadratic program that can be computed with a few passes through the network. We demonstrate a training algorithm using this update, called sequential deactivation (SDA), on MNIST and some synthetic datasets. Upon adopting a natural choice for the nodal weights, SDA has no hyperparameters other than those describing the network structure. Our experiments explore behavior with respect to network size and depth in a family of sparse expander networks. 
\end{abstract}


Some of the simplest modes of machine learning are monotone in character, where the representation of knowledge changes unidirectionally. In the $k$-nearest-neighbor classification algorithm \citep{cover1967nearest}, for example, exemplars of the classes are amassed, monotonically, over the course of training. It is also possible to have monotone learning where memory does not grow, but decreases over time. Suppose we try to classify Boolean feature vectors using conjunctive normal form formulas, one for each class. Starting with large and randomly initialized formulas, during training we simply discard, from the formula of the feature vector's class, all the inconsistent clauses \citep{mooney1995encouraging}. Clearly an attractive feature of both of these monotone algorithms is the computational ease of reaching the desired outcome, i.e. additional exemplars for improved discrimination, or discarded clauses to better accommodate a class.

We report on a new kind of monotone learning where during training only the values of a fixed number of parameters are changed, and in a monotone fashion. Like the two examples above, there is a simple and computationally tractable objective  when learning each new item. But perhaps the most interesting feature of our scheme is that it operates on deep networks of rectifiers, a setting where training is not normally seen as computationally tractable.

Monotone learning with deep rectifier networks is made possible by two changes to the standard paradigm: (i) \textit{conservative} updates that minimize the norm of the parameter changes, and (ii) eliminating all but the bias variables as learned parameters.

Although the feedforward computations in standard neural network models are inspired by biological neuronal computations, the training protocols in widespread use seem far from natural. Humans are able to generalize the shape of the digit 5 without suffering through many training ``epochs" through a data set, and are capable of building a rudimentary representation of digits or other classes of objects from relatively few examples. The conservative learning principle, introduced to the best of our knowledge by \cite{widrow1988layered} as the ``minimal disturbance principle," comes closer to our experience of natural learning. In this mode of learning the parameters of the network are minimally changed to accommodate each example as it is received, with the rationale that the attention to minimality preserves the representation created by earlier examples.

The minimal disturbance principle was reprised by \cite{crammer2006online} as the ``passive-aggressive algorithm", in the context of support-vector machines. There are relatively few models where ``aggressive" parameter updates --- minimal, yet achieving immediate results --- are easily computed. To emphasize the norm-minimizing characteristic of this learning mode we use the term ``conservative learning." An approximate implementation of conservative learning was successfully used to discover the Strassen rules of matrix multiplication \citep{elser2016network}.

In the present work we make changes to the ``standard neural network model" to enable conservative learning. In broad outline, the changes guarantee monotonicity, making the computation of the conservative updates tractable. A key step was elevating the role of the additive parameter, or bias. The metaphor that replaces Hebbian synapse (weight) learning is that of a silting river delta, with myriad channels whose levels (bias) rise differentially, but monotonically, over time. 

Below is an overview of our contributions.

\begin{itemize}
    \item Inspired by analog implementations of logic, we propose replacing the conventional rectified sum computed by a standard neuron using a ReLU function
\[
y\leftarrow \max{\left(0, \;\sum_i w_i x_i-b\right)},
\] 
by a sum of rectifications:
\begin{equation}\label{sumofrect}
y\leftarrow w \sum_i \max{(0,\; x_i-b_i)}.
\end{equation}
Only the bias parameters are learned; the weights are static.

\item Even with just positive weights $w$ in \eqref{sumofrect}, we
  show that networks of \textit{rectified wires} can represent
  arbitrary Boolean functions. This construction makes use of doubled inputs, where \textsc{true} is encoded as $(1,0)$, \textsc{false} as $(0,1)$, and generalizes, for symbolic data, to one-hot encoding.

\item The node values of our networks are non-negative. We propose defining class membership of data by the corresponding output node having value zero.

\item Because the biases on all the wires are non-negative and can only increase during training, we choose to initialize them at zero.

\item We show that the 2-norm minimizing change of the bias parameters, that sets the class output node for a given input to zero, is the solution to a quadratic program.

\item An iterative algorithm similar to stochastic gradient descent, including both backward propagation and
  a new kind of forward propagation,
  is proposed
 to approximately find the 2-norm minimizing bias changes. Executing the minimum number
  of iterations required to make the class output node the smallest in
  value defines the \textit{sequential deactivation} algorithm (SDA).

\item We show that a particular limit of SDA, on networks with a single hidden layer, learns arbitrary Boolean functions.

\item Networks with better scaling have multiple layers, and their training is sensitive to the settings of the static weights. We propose \textit{balanced weights} based just on the in-degrees and out-degrees in this general setting.

\item A two-parameter family of random sparse expander networks is introduced to explore the size and depth behavior of learning in our model.

\item Conservative learning on rectified wire networks with the SDA algorithm is demonstrated for MNIST and synthetic datasets.

\item Despite our model's capacity for depth in its representations, we find that the best experimental results with the SDA algorithm are obtained in a limit where deactivation occurs only in the final layer, where the model is equivalent to a network with just a single hidden layer of rectifier gates. The pursuit of conditions that fully utilize activation/decativation in all layers is identified as the goal of future research.

\end{itemize}

\section{Rectified wire networks}\label{sec:RWnets}

One motivation for our network model is the implementation of logic by analog computations without multiplications. The analog counterparts of \textsc{true} and \textsc{false} are, respectively, the numbers 1 and 0. Before presenting our ``rectified wire" network model, we consider networks constructed from biased rectifier gates. A biased rectifier gate with $K$ inputs,
\[
R[b](x_1,\ldots,x_K)=\max{(0,x_1+\cdots+x_K-b)},
\]
generalizes \textsc{and} and \textsc{or} gates. \textsc{and} gates are realized with bias $b=K-1$. We would get the \textsc{or} gate with $b=0$ if we could saturate the output at the value 1. We show below how this detail can be fixed with additional rectifiers and a suitable network design.

Negations are completely absent from our rectifier implementations of general logic circuits. This is made possible by the process of \textit{demorganization}, where \textsc{not} gates are pushed through the circuit from output to input, exchanging \textsc{and} and \textsc{or} gates as dictated by De Morgan's laws. After all the \textsc{not} gates have been pushed through, the only remaining \textsc{not} gates act directly on the inputs. We accommodate this by allowing each input to the logic circuit to be replaced by an analog pair with values $(1,0)$ for \textsc{true} and $(0,1)$ for \textsc{false}. We refer to this encoding scheme as \textit{input doubling}.

It is relatively straightforward to show, as we do in theorem \ref{thm:thm1}, that networks comprising only rectifier gates, with bias variables as the only parameters, can mimic any logic circuit. From the machine learning perspective, our network model has at least the capacity to represent the classes defined by the truth value of arbitrary Boolean functions. However, the parameter settings that realize these Boolean classes are very special points in the continuous parameter space of the model, and are not claimed to be directly relevant for the intended use of these networks.

\begin{thm}\label{thm:thm1}
    Any Boolean function on $N$ inputs and computed with $M$ binary \textsc{and}/\textsc{or} gates and any number of \textsc{not} gates can be implemented by an analog network comprising at most $5M$ biased rectifier gates taking the corresponding $2N$ doubled analog inputs.
\end{thm}
\begin{proof}
As a circuit, the function takes $N$ Boolean inputs and produces one
Boolean output. We associate negations in the circuit with
gate outputs as shown for the case of the
\textsc{and} gate in the left panel of Figure \ref{fig:fig1}. Upon
output, both the gate output $y_1$ and its negation $y_2$ by a \textsc{not} gate (rendered as an open circle), are made available to
gates receiving inputs. One of the output forks might not be used, as
for the final gate of the circuit whose single output is the value of
the Boolean function. The diagram also shows how each of the
gate inputs is derived from one fork of the output of another gate, or an input to the circuit.

\begin{figure}[t]
\begin{center}
\includegraphics[width=5.in]{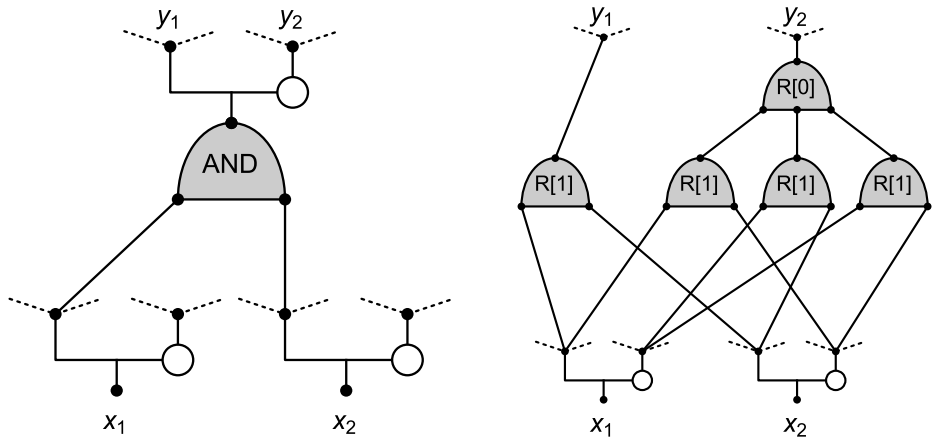}
\end{center}
\caption{\textsc{and} gate and its replacement with rectifier gates in the analog network.}
\label{fig:fig1}
\end{figure}

\begin{figure}[t]
\begin{center}
\includegraphics[width=5.in]{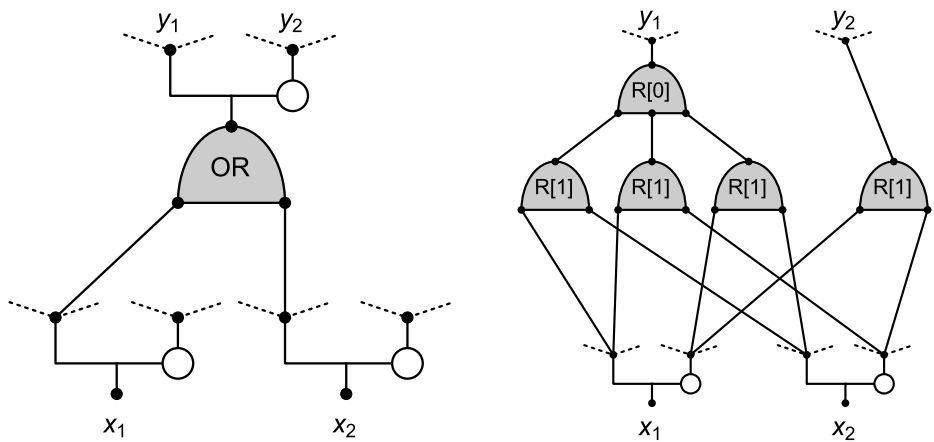}
\end{center}
\caption{\textsc{or} gate and its replacement with rectifier gates in the analog network.}
\label{fig:fig2}
\end{figure}

The body of the proof consists of confirming that the Boolean gate relationship between the inputs $(x_1,x_2)$, and the output $y_1$ and its negation $y_2$, is exactly reproduced by a corresponding rectifier gate network. The replacement for the \textsc{and} gate is shown in the right panel of Figure \ref{fig:fig1}. All bias parameters are either 0 or 1, and the only values that arise on the nodes are also 0 and 1. The replacement rule for the \textsc{or} gate is shown in Figure \ref{fig:fig2}. For both of the circuit replacement rules just described there are three others, where one or both of the inputs are negated. These are trivially generated from the ones shown by inserting a negation right after the inputs ($x_1$ or $x_2$ or both), which is equivalent to swapping the branches of the input-forks.

For completeness we include the cases where a gate receives both
inputs from the same output fork. The case of the
\textsc{and} gate is shown in Figure \ref{fig:fig3}. The replacement
rule shown on the left applies when both inputs come from the same
branch of the input fork, while the rule on the right applies when the inputs come from both branches. As before, the rule for the other choice of input branch is obtained by swapping branches in the replacement. For the \textsc{or} gate the rule on the left of Figure \ref{fig:fig3} is unchanged, while the one on the right has the two bias values swapped.

\begin{figure}[t]
\begin{center}
\includegraphics[width=5.5in]{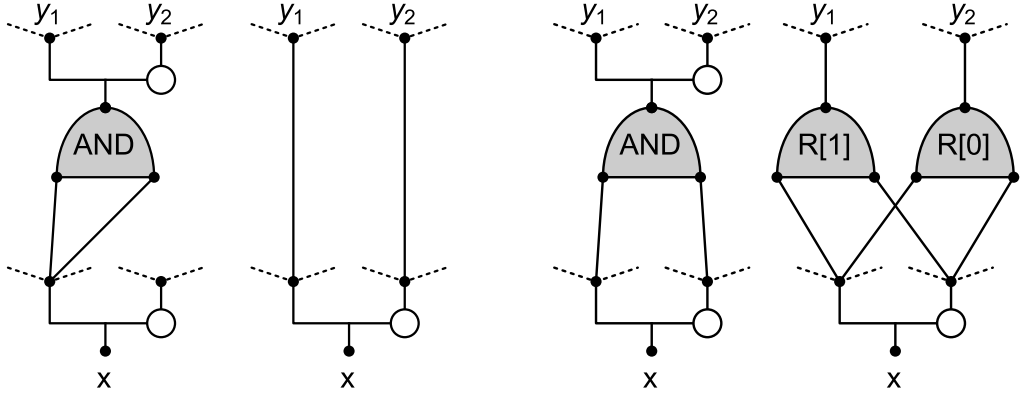}
\end{center}
\caption{\textsc{and} gate replacement rules when both inputs come
  from the same output fork, either the same branch (left) or both branches (right). The rule for the \textsc{or} gate is the same for the case on the left, and has $R[0]$ and $R[1]$ exchanged for the case on the right.}
\label{fig:fig3}
\end{figure}

The only negations remaining after all \textsc{and} and \textsc{or} gates have been replaced are those in the $N$ input forks to the circuit. These are taken care of by defining the input to the analog network as the $2N$ values on the branches of those forks. The resulting network will have only rectifier gates. From the gate replacement rules (Figs. \ref{fig:fig1}-\ref{fig:fig3}) we see that the number of rectifier gates is at most $5M$.

\end{proof}

We now generalize our rectifier gate model in two ways. The first follows from
the observation that a rectifier on $K$ inputs with bias $b$ is
equivalent to summing those inputs and placing a (single-input) rectifier with bias $b$ on each of the wires leaving the rectifier. We give these biases the freedom to take different values, increasing the expressivity of the network and boosting the parameter count commensurate with ``standard model" networks, which have a weight on each edge. By setting the biases on selected edges in a completely connected network to large values, thereby disabling them, our model has the capacity to select its architecture.

The second way we generalize the model is to introduce weights at all the summations. These are static --- not intended to be learned --- and positive. We show in section \ref{sec:weights} that in layered networks one can apply a rescaling to the biases that has the effect of restoring all the weights to 1 without changing the classification behavior of the network. The weights therefore only play a role during training, through the gain/decay they impart to the analog signal as it propagates from input to output, and the response this elicits from the bias parameters.

We now formally define our model.

\begin{defn}\label{def:rectwire}
A \textit{rectified wire network} is a model of computation on a directed acyclic graph. Associated with each wire or edge $i\to j$ joining node $i$ to $j$ is the node value $x_i$ and edge output $y_{i\to j}$ related by
\[
y_{i\to j}=\max{(0,x_i-b_{i\to j})},
\]
where $b_{i\to j}$ is the bias parameter for the edge. The edge outputs $y_{i\to j}$, from all edges $i\to j$ incident on node $j$, are summed to give the value $x_j$ of node $j$:
\[
x_j=w_j\sum_{i\to j}y_{i\to j}.
\]
The positive constant $w_j$ is the \textit{weight} of the node. The node values $x$, edge outputs $y$, and bias parameters $b$ of a rectified wire network are general real numbers.

The number of edge outputs $y$ summed by a node $j$ is its \textit{in-degree} $|{\to}j|$, and the number of edges receiving $x_j$ as input is its \textit{out-degree} $|j{\to}|$.
Nodes with in-degree zero are \textit{input nodes} and have their $x$ values assigned. The outputs of the network are the $x$ values of the nodes with out-degree zero, the \textit{output nodes}. Nodes which are neither input nor output nodes are the \textit{hidden nodes} of the network. The connectivity of rectified wire networks is such that there always exists a path between any input node and any output node. 
\end{defn}

The content of theorem \ref{thm:thm1}, re-expressed in terms of rectified wires, takes the following form:
\begin{cor}\label{cor:cor1}
The truth value of a Boolean function on $N$ variables that can be computed with $M$ binary \textsc{and}/\textsc{or} gates and any number of \textsc{not} gates can be represented by a two-output rectified wire network taking $2N$ doubled inputs and having at most $7M$ hidden nodes. Like the doubled inputs, the two output nodes encode the truth of the function with values $(1,0)$ or $(0,1)$. 
\end{cor}
\begin{proof}
A logic circuit with $M$ binary \textsc{and}/\textsc{or} gates has $M$ nodes, one at each gate output. When re-expressed as a rectified wire network, each gate output is replaced by two nodes and, in the worst case (Figs.~\ref{fig:fig1}-\ref{fig:fig2}), the gate itself results in five additional nodes. The number of hidden nodes in the resulting rectified wire network is therefore bounded by $7M$. The wiring of the nodes, including the input and output nodes, can be the complete acyclic graph because large bias settings, by deselection, realize any network on the given number of nodes.
\end{proof}

While possible in principle, we should not expect it to be easy for a rectified wire network to learn bias values on the modest networks promised by corollary \ref{cor:cor1} when presented with data generated by a Boolean function. On the other hand, as we show in section \ref{sec:class}, this particular network model has a significant advantage over the ``standard model" in having tractable conservative learning. A different, equally valid representation might be learned instead, most likely on a much larger network.

We render rectified wire networks as simple directed graphs, with input nodes at the bottom, output nodes at the top, and all edges directed in the upward direction. The rendering of the bias values --- the only learned parameters --- is through the edge thickness. Edges with low bias are thick, those with high bias are thin. When the bias on an edge is so great that the rectified value is zero for any input to the network, the edge is effectively absent (zero thickness). Figure \ref{fig:fig4} shows this style of rendering of a network with four input, four hidden, and two output nodes. 

\begin{figure}[!t]
\begin{center}
\includegraphics[width=3.in]{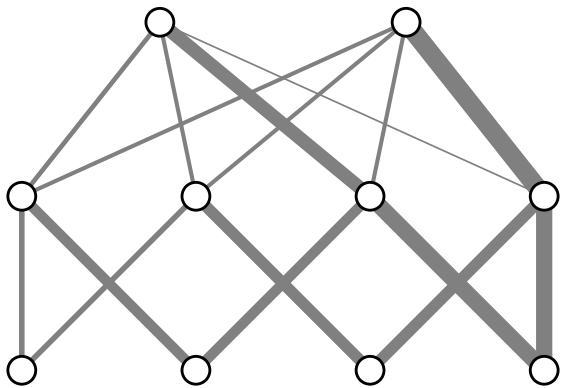}
\end{center}
\caption{Rectified wire network with four input nodes (bottom row), four hidden nodes, and two output nodes (top row). Thinnest edges have the highest bias.}
\label{fig:fig4}
\end{figure}

We can generalize both the inputs and the outputs of rectified wire networks to expand their scope beyond Boolean function-value classifiers. In many applications the feature vectors are strings of symbols from a finite alphabet: $\{$\textsc{true, false}$\}$, $\{$\textsc{a,c,g,t}$\}$, etc. As in the Boolean case, for an alphabet of size $K$ we use the one-hot encoding scheme where the $k^\mathrm{th}$ symbol is expanded into $K$ numerical values all zero except for a 1 in the $k^\mathrm{th}$ position. Numerical (non-symbolic) data may be encoded in the same style. Suppose the data vectors have length $L$: $v=(v_1,\ldots,v_L)$. We first obtain the cumulative probability function $\theta_i$ for each component $i=1,\ldots,L$ over the dataset. The data vector for input to the network, corresponding to $v$, is then
\begin{equation}\label{analogvec}
d(v)=\left[\;\theta_1(v_1)\,,1-\theta_1(v_1)\;,\; \ldots \;,\; \theta_L(v_L)\,,1-\theta_L(v_L)\;\right].
\end{equation}
With this convention, the input to the network is always a vector of numbers between 0 and 1 whose sum equals the number of components or  string length of the data. The mapping by the cumulative probability functions, in the case of analog data, ensures that the distribution in each  channel is uniform.

The number of output nodes equals the number of data classes. A Boolean function classifier would have two output nodes to classify the truth value of the input. To classify MNIST data \citep{lecun1998gradient}, the network would have 10 output nodes. Before we describe how these outputs are interpreted we make a trivial observation:
\begin{lem}\label{lem:indicate}
In a rectified wire network with input vector $d\ge 0$, there exist settings of the bias parameters where, for any of the output nodes $c\in C$, we have $x_c=0$ and $x_j>0$ for $j\in C\setminus\{c\}$.
\end{lem}
\begin{proof}
All the output nodes will have positive values for zero bias, because the weights are positive and each output is connected by a path to one of the inputs with positive value (see definition \ref{def:rectwire}). To exclusively arrange for $x_c=0$, we set $b_{i\to c}=x_i$ on all the edges $i\to c$ incident on $c$ (keeping all other biases at zero).
\end{proof}
We will explain in section \ref{sec:class} why zero is the appropriate indicator function value for defining classes on a conservatively trained rectified wire network.

\section{Notation}

Now that most of the elements of our network model have been
introduced, we summarize our notational conventions. Lower case
symbols represent variables or parameters on the network. These bear a
latin index when defined on a node of the network. Examples are the
node values, $x_i$, or the weights $w_i$. The index is written as a
directed edge between nodes when defined on an edge of the
network. Examples are the bias parameters, $b_{i\to j}$, and
edge-outputs $y_{i\to j}$. The same lower case symbols, both for nodes
and edges, when written without subscripts denote the vector of all
the indexed values.  For example, $b$ is the vector of all biases in
the network and $d$ the vector of data values on the input
nodes. When such symbols appear in an equality/inequality, the relation applies componentwise. Round brackets attached to a variable identify its functional
dependence, when relevant. The value of node $k$, for example,
might be written $x_k(b,d)$ to emphasize its dependence on the biases and input data.

Upper case symbols are used for sets. The symbols $D$, $H$, $C$ are reserved for the data (input), hidden, and class (output) nodes respectively, while $E$ is the set of directed edges. Upper case symbols are also used for continuous sets with the same convention for functional dependence as variables. For example, $B(d)$ might be the set of bias values $b$ such that for $b\in B(d)$ the node value $x_i(b,d)$ is zero for input $d$. Cardinality is denoted with vertical bars, e.g. $|E|$ for the number of edges in the network. Node $i$ in a network has in-degree $|{\to} i|$ and out-degree $|i{\to}|$. 

\section{Counterfactual classification and conservative learning}\label{sec:class}

Monotone learning on rectified wire networks is made possible by the fact that, for a fixed input $d\ge 0$ to the network, the value $x_k(b,d)$ at any node $k$ is a non-negative and non-increasing function of the bias parameters $b$. This, combined with the fact (lemma \ref{lem:indicate}) that zero is always a feasible value on any of the output nodes, motivates the following definition of a classifier.

\begin{defn}\label{def:classrule}
The output nodes $C$ of a rectified wire network serve as a classifier of $|C|$ classes of inputs when there are biases $b$ such that for inputs $d$ in class $c\in C$, $x_c(b,d)=0$ and $x_i(b,d)>0$ for $i\in C\setminus\{c\}$. 
\end{defn}

To see why this definition makes sense in the context of conservative learning, suppose that at some stage of training all the output nodes are positive for some input $d$ in class $c$. By increasing the biases $b$, the output value $x_c(b,d)$ can always be driven to zero for input $d$, and by increasing them as little as possible, i.e.~conservatively, there is a good chance the other output nodes can be kept positive. If some of the other outputs are also set to zero in the process of increasing biases, the classification is ambiguous. In any case, and after any amount of subsequent training where biases are increased, monotonicity of $x_c(b,d)$ with respect to $b$ ensures $x_c(b,d)=0$ continues to be a valid indicator for class $c$.

The outputs of the proposed classifier are counterfactual in the sense that a large positive value $x_c$ on the node for class $c$ represents strong evidence that $c$ is not the correct class. In the case of a class defined by the truth value of a Boolean function, we would use the function that returns \textsc{false} for members of the class, so the value of the corresponding analog circuit output node is zero.

We shall see that we do not need to separately insist that, during training, biases are never decreased, as it will follow automatically from the conservative learning principle which assigns the following cost when biases are changed:
\begin{defn}
The cost, in a rectified wire network, for changing the biases from $b$ to $b'$, is the $p$-norm $\|b'-b\|_p$.
\end{defn}

The uniqueness of the cost-minimizing bias change, for fixing a misclassification, depends on the norm exponent $p$. We mostly use $p=2$ in this study, where uniqueness can be proved. Our definition of the conservative learning rule, which sidesteps uniqueness, is the following:

\begin{defn}
Suppose output node $c\in C$ for class $c$ has value $x_c(b,d)>0$ for an input $d$, in a network with biases set at $b$. A conservative bias update is any $b'$ that minimizes $\|b'-b\|_p$ and produces output value $x_c(b',d)=0$.
\end{defn}

\begin{lem}\label{lem:lem1}
If $b'$ is any conservative update of $b$, then $b'\ge b$.
\end{lem}
\begin{proof}
Biases $b'$ have just learned an input $d$ in some class $c\in C$, so $x_c(b',d)=0$. Now suppose $b'_{i\to j}<b_{i\to j}$ for some edge $i\to j\in E$. The update $b''$, where $b''_{i\to j}=b_{i\to j}$ and all others are unchanged, has lower cost than $b'$ and yet, by monotonicity of node values with respect to $b$, all output nodes with value zero will still be zero, including $x_c(b'',d)$. 
\end{proof}

The outputs of a rectified wire network are compositions of rectifier functions. Composition of convex functions are in general not convex, but for the rectifier function this is true. We give an inductive proof based on the following.

\begin{lem}\label{lem:convlemma}
If $x(b):\mathbb{R}^K\to\mathbb{R}$ is convex, then $y(a,b):(\mathbb{R},\mathbb{R}^K)\to\mathbb{R}$, defined by
\[
y(a,b)=\max{(0,x(b)-a)},
\]
is also convex.
\end{lem}
\begin{proof}
Because $x$ is convex, for arbitrary $b_1,b_2\in\mathbb{R}^K$ we have
\begin{equation}\label{xconv}
x(\lambda b_1+(1-\lambda)b_2)\le \lambda\, x(b_1)+(1-\lambda)\,x(b_2)
\end{equation}
for $0\le\lambda\le 1$. For arbitrary $a_1,a_2\in\mathbb{R}$,
\[
\begin{split}
y\left(\lambda a_1+(1-\lambda)a_2,\right.&\left.\lambda b_1+(1-\lambda)b_2\right)\\
&=\max{\left(0,x(\lambda b_1+(1-\lambda)b_2)-\lambda a_1-(1-\lambda)a_2\right)}\\
&\le\max{\left(0,\lambda\, x(b_1)+(1-\lambda)\,x(b_2)-\lambda a_1-(1-\lambda)a_2\right)}\\
&=\max{\left(0,\lambda\,z_1+(1-\lambda)\,z_2\right)},
\end{split}
\]
where we have used \eqref{xconv} and defined
\begin{align*}
z_1&=x(b_1)-a_1\\
z_2&=x(b_2)-a_2.
\end{align*}
Since for arbitrary $s,t\in\mathbb{R}$
\[
\max{(0,s+t)}\le\max{(0,s)}+\max{(0,t)},
\]
we obtain
\begin{align*}
\max{\left(0,\lambda\,z_1+(1-\lambda)\,z_2\right)}&\le\max{(0,\lambda\,z_1)}+\max{(0,(1-\lambda)\,z_2)}\\
&=\lambda\max{(0,z_1)}+(1-\lambda)\max{(0,z_2)}\\
&=\lambda\,y(a_1,b_1)+(1-\lambda)\,y(a_2,b_2)
\end{align*}
as claimed.
\end{proof}

\begin{thm}\label{thm:convbias}
The node values $x(b,d)$ of a rectified wire network, given an input $d$, are convex functions of the bias parameters $b$.
\end{thm}
\begin{proof}
The values of the input nodes $D$, set to the constant values $x_i=d_i, i\in D$, are trivially convex. Since our network is an acyclic graph, the nodes can be indexed by consecutive integers so that all edge outputs $y_{i\to j}$ incident on a node $j\in H\cup C$ are from input nodes $i\in D$ or from hidden nodes $i\in H$ that have $i<j$. For any $j\in H\cup C$, we use induction and suppose all $x_i$, $i\in H$ with $i<j$ are convex functions of the bias parameters. Consider any of the edge outputs $y_{i\to j}(b)$ incident on $j$:
\[
y_{i\to j}(b)=\max{(0,x_i(b)-b_{i\to j})}.
\]
Seen as a function of bias parameters, where $x_i(b)$ is a convex function of biases not including $b_{i\to j}$, $y_{i\to j}(b)$ fits the hypothesis of lemma \ref{lem:convlemma} and is therefore convex. This conclusion also applies to the base case, the hidden node numbered $j=1$, where all the $x_i$ are inputs $i\in D$ and constant. Since the node value
\[
x_j(b)=w_j\sum_{i\to j}y_{i\to j}(b)
\]
is a positive multiple of a sum of convex functions, it too is convex and completes the induction.
\end{proof}

For the case $p=2$, we can prove uniqueness of the conservative bias update.
\begin{cor}
The set $B(d)$, of biases $b$ of a rectified wire network with fixed
input $d$ for which $x_c(b,d)=0$ for a given output node $c$, is
non-empty and convex. Consequently, the conservative bias update
$b'\in B(d)$ that minimizes the cost $\|b'-b\|_2$ with respect to the
biases $b$ is unique.
\end{cor}
\begin{proof}
The set $B(d)$ may be defined as the feasible set of $x_c(b,d)\le 0$ (since this function is non-negative). By lemma \ref{lem:indicate} $B(d)$ is non-empty, and by theorem \eqref{thm:convbias} $B(d)$ is closed and convex because $x_c(b,d)$ is a convex function of $b$. Now suppose $b'_1\in B(d)$ and $b'_2\in B(d)$ are distinct minimizers of $\|b'-b\|_2$ for some bias $b$. Because $B(d)$ is convex we have a contradiction because $b'_3=(b'_1+b'_2)/2\in B(d)$ and $\|b'_3-b\|_2<\|b'_1-b\|_2=\|b'_2-b\|_2$ implies $b'_1$ and $b'_2$ were not minimizers.
\end{proof}

To design algorithms for computing the conservative bias update and
to establish the complexity of this task we cast the problem as a mathematical program. Consider a rectified wire network with edges $E$ on which we are given bias parameters $b$, say from previous training. Suppose we now wish to learn the pair $(c,d)$, a vector of input data $d$ in the class associated with output node $c$. Unless $x_c(b,d)=0$, the biases must be conservatively updated to $b'$ so this is true. Before we do this we partition $E$ into the set of \textit{active} edges $A$ and its complement, $\bar{A}$.

\begin{defn}
The set $A$ of active edges of a rectified wire network, given input data $d$ and bias parameters $b$, are those edges $i\to j$ where $x_i(b,d)>b_{i\to j}$, or equivalently, where $y_{i\to j}(b,d)>0$.
\end{defn}

The biases $b_{i\to j}$ on the inactive edges $i\to j\in\bar{A}$, when increased, have no effect on any node values, in particular the output nodes, because the corresponding edge-outputs $y_{i\to j}$ are zero by monotonicity of $x(b)$ (non-increasing). To find $b'$ we may therefore work with the network induced by the active edges $A$, called the ``active network."

Our mathematical program makes use of a set of \textit{reduced bias} variables.
\begin{defn}
Let $A$ be the active edges of a rectified wire network for input data $d$ and bias parameters $b$. For each edge $i\to j$ of the network, define the reduced bias by
\[
b^-_{i\to j}=\left\{
\begin{array}{rl}
b_{i\to j},& {i\to j}\in A\\
x_i(b,d),&{i\to j}\in \bar{A}.
\end{array}
\right.
\]
\end{defn}
Now consider a conservative update $b'$ of $b$, and the reduced biases $b^-$ for $b'$. We always have $b^-\le b'$, and the reduced values have the property that $x(b^-,d)=x(b',d)$. In particular, for the class output node $c$ we have $x_c(b^-,d)=x_c(b',d)=0$. The only reason that a conservative update $b'$ of $b$ would not already be reduced ($b'=b^{-}$) is that the freedom $b^-\le b'$ may allow a reduction in cost, that is, give a more conservative update.

We now define the mathematical program. For given data $d$, biases $b$ and weights $w$ on a rectified wire network, we are given the corresponding active edges $A$ and node $c$ of the correct class. The unknowns are the updated biases $b'$, reduced biases $b^-$, and edge-outputs $y$ on $A$, and the values of $x$ on $c$ and the input and hidden nodes:
\begin{subequations}
\label{MP}
\begin{align}
\mbox{minimize:  }&\|b'-b\|_p^p&&\label{mathprog1}\\[10pt]
\mbox{such that:  }&b^-\le b',&&\label{mathprog2}\\
&0\le b^-,&&\label{mathprog3}\\
&x_i=d_i,&&i\in D\label{mathprog4}\\
&y_{i\to j}=x_i-b^-_{i\to j},&&i\to j\in A\label{mathprog5}\\
&0\le y_{i\to j},&&i\to j\in A\label{mathprog6}\\
&x_j=w_j\sum_{i\to j\in A}y_{i\to j},&&j\in H\cup\{c\}\label{mathprog7}\\
&x_c=0.\label{mathprog8}&&
\end{align}
\end{subequations}

\begin{thm}\label{thm:MP}
The mathematical program \eqref{MP}, defined for the network $A$ that is active for given biases $b$ and data $d$, as part of its solution $b'$ gives a conservative bias update for the class associated with output node $c$.
\end{thm}
\begin{proof}
First observe that any valid ($x_c=0$) assignment of node values $x$ and edge-outputs $y$, in the corresponding rectified wire network, is  realized by a feasible point of the linear system \eqref{mathprog3}~--~\eqref{mathprog8} involving just the reduced biases (not $b'$). For the variables $b'$  to be rectified wire biases consistent with the reduced biases $b^-$, it is sufficient for them to satisfy \eqref{mathprog2}, as monotonicity ensures $x_c(b')=0$. Optimizing \eqref{mathprog1} subject to \eqref{mathprog2} and the constraints that define the feasible set of reduced biases, gives the most conservative update.
\end{proof}

For $p=2$, \eqref{MP} is a positive semi-definite quadratic program
and can be solved in time that grows polynomially
\citep{kozlov1980polynomial} in the size of the active network, $|A|$. The same conclusion applies for $p=1$, a linear program, because the objective can be replaced by the sum of all the updated biases $b'$ provided we impose the result of lemma \eqref{lem:lem1}, $b\le b'$, as a constraint. While these results allow us to add rectified wire networks to the list of models for which there is a tractable conservative learning rule, in most applications it is also important that training scales nearly linearly with the size of the network. The \textit{sequential deactivation algorithm} (\textsc{SDA}) described in the next section is designed to meet that goal.

We close this section by reviewing the features of our model that make training tractable. The insistence on exactly learning individual items should not be counted as a tractability-enabling feature. Indeed, it is easy to see how the mathematical program \eqref{MP}, to compute the conservative update for one data item, would be generalized to find the norm-minimizing update that classifies an entire mini-batch. All the variables and data, with the exception of the biases $b$ and $b'$, now carry a data index $m\in M$, where $M$ is a mini-batch. For example, \eqref{mathprog2} would be replaced by
\[
{b^-}_{i\to j, m}\le {b'}_{i\to j}\qquad i\to j\in A,\; m\in M
\]
and imposes consistency of the bias updates for all instantiations of the reduced biases over the mini-batch. While the size of the mathematical program has grown to $O(|A||M|)$ equations and inequalities, it is still tractable.

Convexity of the variables $x$ and $y$ with respect to the bias parameters is clearly important for tractable learning, and the proofs of lemma \ref{lem:convlemma} and theorem \ref{thm:convbias} show that this relies on the non-negativity of the static weights and the form of the activation function. Less obvious in facilitating tractability is the proposal to replace the loss function of neural networks by a constraint, and in particular, the simple constraint \eqref{mathprog8}. For example, one might consider replacing this single constraint with the following set of inequalities,
\begin{equation}\label{marginconstraint}
x_i>x_c+\Delta,\quad i\in C\setminus {c},
\end{equation}
where the fixed positive parameter $\Delta$ specifies a margin for avoiding ambiguous classification. After all, we still have a tractable mathematical program after this substitution. Unfortunately, this proposal creates a conflict in the relationship between the reduced biases $b^-$ and the actual bias updates $b'$. While the $b^-$ variables provide an exhaustive parameterization of the node variables $x$, even under the constraint \eqref{marginconstraint}, we cannot count on the inequality $b^-\le b'$ to preserve this constraint when $b^-$ biases are replaced by $b'$ biases in the rectified wire model (the monotonic decrease of a non-class output may exceed that of the class output).

\section{Sequential deactivation algorithm}\label{sec:sda}

This section describes in detail the \textsc{SDA} algorithm for computing bias updates in the rectified wire model. While these updates are not guaranteed to be the most conservative possible, their computation is significantly faster than solving a general quadratic program. The SDA algorithm resembles the gradient methods used for optimizing standard model networks, the most widely used being stochastic gradient descent (\textsc{SGD}) \citep{rumelhart1986learning}.

The program \eqref{MP} would be greatly simplified if inequality \eqref{mathprog6} was somehow automatically satisfied. All active edges would remain active in the update and the updated biases $b'$ on them would be equal to the reduced biases $b^-$. The node variables $x$, in particular $x_c$, would be linear functions of the $b'$:
\[
x_c(b')=x_c(b)+\nabla x_c\cdot (b'-b).
\]
For the $p=2$ norm on bias updates, the optimal update that achieves $x_c=0$ is then
\[
b'=b-\frac{x_c(b)}{\|\nabla x_c\|^2}\nabla x_c.
\]
The conservative update, in this simplification, is the same as gradient descent.

In the general case, where edges may become inactive, from the monotonicity of $x_c$ we know at least that the gradient $\nabla x_c$ never has positive components. If we had the stronger property, that $\nabla x_c$ is also never zero while $x_c>0$, then gradient descent will always find an update $b^*$ that satisfies $x_c(b^*)=0$. But the latter property follows immediately from the fact that $x_c>0$ is only possible when there are active edges incident to output node $c$, so that  $x_c$ has non-zero derivatives at least with respect to the biases on them. 

While gradient descent cannot guarantee the norm minimizing update promised by the solution of program \eqref{MP}, it has some attractive features. First, the computation of the (piecewise constant) gradient is fast, as is the computation of the step size to the next gradient discontinuity (deactivation event). Second, there is an opportunity to make the gradient descent update $b^*$ more conservative on the inactive edges. Using the condition $y_{i\to j}(b^*)=0$ as the indicator of an inactive edge, the improved update is, for all $i\to j\in A$,
\begin{equation}\label{improvedupdate}
b'_{i\to j}=\left\{
\begin{array}{rl}
\max\left(b_{i\to j}, x_i(b^*)\right), & y_{i\to j}(b^*)=0\\
b^*_{i\to j}, &y_{i\to j}(b^*)>0.
\end{array}
\right.
\end{equation}
The SDA algorithm is the efficient implementation of gradient descent on the rectified wire model followed by \eqref{improvedupdate}. Its name draws attention to the fact that the steps in the descent are defined by deactivation events.

We now turn to the implementation of the SDA algorithm. The elementary operations are defined in Algorithm \ref{alg1}. Procedure \textsc{eval} is simply a forward-pass through the network, from data supplied at the input nodes, to the output nodes whose values directly encode the class. In addition to the output node values $\{x_i: i\in C\}$, \textsc{eval} also provides the edge outputs $y$ and the set of active edges $A$ (on which the edge outputs are non-zero).

Suppose the current bias parameters are $b(0)$. Denote by $(-\nabla x_c)_{i\to j}$ the negative gradient component for the bias on edge $i\to j\in A$. Because the derivatives with respect to all the $b_{i\to j}$ on active edges into the same node $j$ are equal, the node-indexed variables
\[
(-\nabla x_c)_j\coloneqq (-\nabla x_c)_{i\to j},\quad i\to j\in A
\]
are well defined. In particular, if $j$ is an output node, then $(-\nabla x_c)_j=w_c\,\delta_{j c}$. As long as the gradient is constant, the biases evolve as
\begin{equation}\label{biasevolve}
b_{i\to j}(t)=b_{i\to j}(0)+t (-\nabla x_c)_j,
\end{equation}
where $t$ is a continuous ``time".

The gradient $(-\nabla x_c)_i$ is positive only if node $i$ is connected by active edges to node $c$, with value given by the sum over all paths on active edges, each contributing by the product of weights along the path. The effect on $x_c$ of a bias change on an active edge into $i$ is the same as if the same bias change was instead applied to all the active edges leaving node $i$, but multiplied by $w_i$. This implies the recursion,
\begin{equation}\label{gradrecursion}
(-\nabla x_c)_i= w_i\sum_{i\to j\in A}(-\nabla x_c)_j,\quad i\in H
\end{equation}
in the procedure \textsc{grad} of Algorithm \ref{alg1} and corresponds to back propagation in the network.

\algrenewcommand{\algorithmiccomment}[1]{\hfill#1}
\begin{algorithm*}[t!]
	\caption{Elementary network procedures}
	\begin{algorithmic}[0]
		\Procedure{eval}{$E$, $b$, $d$} $\rightarrow (A, x, y)$
		\State $x_i\leftarrow d_i,\quad i\in D$
		\State $x_j\leftarrow w_j\sum_{i\to j\in E}\;y_{i\to j},\quad j\in H\cup C$
		\State $y_{i\to j}\leftarrow \max(0,x_i-b_{i\to j}),\quad i\to j\in E$
		\Statex
		\State $A\leftarrow \{i\to j\in E: y_{i\to j}=0\}$
		\EndProcedure
		\Statex
		\Procedure{grad}{$A$, $c$} $\rightarrow \nabla x_c$
		\State $(-\nabla x_c)_j\leftarrow w_c\,\delta_{j c},\quad j\in C$
		\State $(-\nabla x_c)_i\leftarrow w_i\sum_{i\to j\in A}(-\nabla x_c)_j,\quad i\in H$
		\EndProcedure
		\Statex
		\Procedure{velocity}{$A$, $\nabla x_c$} $\rightarrow \dot y$
		\State $(-\dot{y})_{i\to j}\leftarrow (-\nabla x_c)_j+w_i\sum_{k\to i\in A}\; (-\dot{y})_{k\to i},\quad i\to j\in A$
		\EndProcedure
		
	\end{algorithmic}\label{alg1}
\end{algorithm*}

The third procedure, \textsc{velocity}, is derived from the recursion for the edge-outputs $y$:
\[
y_{i\to j}=w_i\sum_{k\to i\in A} y_{k\to i}-b_{i\to j}.
\]
Taking the time derivative and using \eqref{biasevolve},
\begin{equation}\label{velrecursion}
(-\dot{y})_{i\to j}=w_i\sum_{k\to i\in A} (-\dot{y})_{k\to i}+(-\nabla x_c)_j
\end{equation}
we obtain the velocities of the edge-outputs by forward propagation.
The initialization occurs at edges $i\to j$ from all the input nodes $i$, for which the sum in \eqref{velrecursion} is absent and $(-\nabla x_c)_j$ is set by \textsc{grad}.

By construction, both $y_{i\to j}(0)$ and the (constant) velocities $(-\dot{y})_{i\to j}$ on the active edges $i\to j\in A$ are positive and the first deactivation event occurs at time
\begin{equation}\label{stepsize}
t^*=\min_{i\to j\in A}\;\frac{y_{i\to j}(0)}{(-\dot{y})_{i\to j}}.
\end{equation}
This is the time step in one iteration of the SDA algorithm. After the biases are incremented by \eqref{biasevolve} and the newly deactivated edges are removed from $A$, another round is begun. Iterations are terminated when \textsc{eval} returns $x_c=0$. The final biases are obtained by applying \eqref{improvedupdate} to the gradient descent biases.

We now revisit the classification rule of definition \ref{def:classrule}, where data is declared learned when $x_c=0$. Recall that this has the desired property of not having the ``supremacy" (smallest value among all classes) of node $c$ spoiled by subsequent training. The only way that the latter can have a negative effect, through the general decrease in output values with increasing biases, is when output nodes other than $c$ are in a tie at value zero. To mitigate this effect, and also make the bias changes even more conservative, we introduce the \textit{ultra-conservative learning} rule:
\begin{defn}\label{def:ultracon}
In the ultra-conservative mode of learning with the \textsc{SDA} algorithm, iterations are terminated when $x_c$, the value of the output node $c$ of the correct class, is either zero or smaller than the values of all the other output nodes.
\end{defn}
This termination criterion is conservative from the point of view of testing. Consider the early stages of training, when we might want to test the algorithm on which class is being favored. The natural candidate for the latter is the output node with the smallest value. By terminating \textsc{sda} as soon as this criterion is met, the test (when performed right after training) will succeed with a smaller change to the biases than is required by the $x_c=0$ criterion.

Because the property of data $d$ producing the smallest output value on node $c$ can be spoiled by training with data $d'\neq d$, it may be necessary for the network to retrain on $d$, or data similar to $d$, in order to properly learn the combination $(d,c)$. The implied heuristic is that the net bias change for this mode of learning class $c$ may be more conservative than always insisting on output $x_c=0$ for every $d$ in this class. 

The \textsc{sda} algorithm with the ultra-conservative termination criterion is summarized in Algorithm \ref{alg2}. As soon as $x_c$ becomes the smallest output or is zero, the algorithm returns the most conservatively updated biases $b'$ consistent with the activity of the network edges, \eqref{improvedupdate}. Estimating the work required to correct a misclassification is complicated by two factors. While the work in one iteration scales linearly with the size of the active network, we do not know how many iterations are needed on average to satisfy the ultra-conservative termination rule. Empirically (section \ref{sec:exp}) we find the number of iterations is often quite small and depends only weakly on network size. A second effect, serving to lessen the work, is the sparsification of the active network with time. In any event, the total computation in learning each data item would be at most $O(|A|^2)$, since in each iteration at least one of the active edges is deactivated (and eliminated from subsequent rounds).

\algrenewcommand\algorithmicrequire{\textbf{input}}
\algrenewcommand\algorithmicensure{\textbf{output}}
\begin{algorithm*}[t!]
	\caption{Sequential Deactivation (\textsc{SDA})}
	\begin{algorithmic}[0]
		\Require $(E,b^0), (d,c)$\Comment{(network edges, initial biases), (data vector, class)}
		\Statex
		\State $b\leftarrow b^0$ \Comment{initialize biases}
		\State $(A,x,y)\leftarrow \mbox{\textsc{eval}}(E,b,d)$\Comment{active edges, node and edge output values}
		\State $\mathsf{iter}\leftarrow 0$\Comment{zero the iteration counter}
		\Statex
		\While{$x_c>0$ and $\argmin_{i\in C} x_i\ne\{c\}$}\Comment{iterate until node $c$ is smallest}
		\State $\nabla x_c\leftarrow \mbox{\textsc{grad}}(A,c)$\Comment{gradient with respect to active biases}
		\State $\dot{y}\leftarrow\mbox{\textsc{velocity}}(A,\nabla x_c)$\Comment{edge output velocity}
		\State $t^*\leftarrow\min_{i\to j\in A}\;y_{i\to j}/(-\dot{y}_{i\to j})$\Comment{step size}
		\Statex
		\For{$i\to j\in A$}
		\State $b_{i\to j}\leftarrow b_{i\to j} + t^*\,(-\nabla x_c)_j$\Comment{increase biases}
		\EndFor
		\Statex
		\State $(A,x,y)\leftarrow \mbox{\textsc{eval}}(A,b,d)$\Comment{new active edges, node and edge output values}
		\State $\mathsf{iter}\leftarrow\mathsf{iter}+1$\Comment{increment iteration counter}
		\EndWhile
		\Statex
		\For{$i\to j\in E$}
		\If{$i\to j\in A$}
		\State $b'_{i\to j}=b_{i\to j}$\Comment{keep biases on active edges}
		\Else
		\State $b'_{i\to j}=\max(b_{i\to j}^0,x_i)$\Comment{conservative update for inactive edges}
		\EndIf
		\EndFor
		\Statex
		\Ensure $b'$, $\mathsf{iter}$\Comment{updated biases, iteration count}
		\end{algorithmic}\label{alg2}
\end{algorithm*}

A possible direction for future research is the design of an algorithm that solves the program \eqref{MP}, for the most conservative update, with only a modest amount of extra work than used by SDA. In this connection we present the simplest instance we know of where the update found by SDA is non-optimal. Figure \ref{fig:fragment} shows a fragment of a network where the bias updates on four edges can be analyzed in isolation. We consider the original optimization problem \eqref{MP}, where the bias updates must give zero on node $c$. There may be additional edges incident on the nodes shown, but these are inactive for the data item being learned. Input node 1 has been zero for all previous data and explains why biases $b_{1\to 2}$, $b_{2\to 3}$ and $b_{2\to c}$ are still zero. The inactive edge incident on node 3 was active on previous data and accounts for the positive bias $b_{3\to c}=b$. The nodes all have unit weight except node 2, whose weight is $w$. Data vector component $d_1=1$ (on node 1) is the only one relevant for the fragment shown.

\begin{figure}[!t]
\begin{center}
\includegraphics[width=2in]{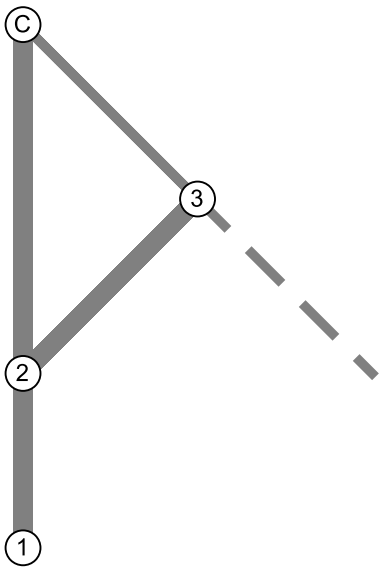}
\end{center}
\caption{A network fragment on which the SDA update is not optimal. The initial bias on the thick edges is zero, while $b_{3\to c}=b$. Another parameter is the weight $w$ of node 2; all other nodes have unit weight. The dashed edge is inactive initially and irrelevant, but was responsible for $b>0$ from prior training.}
\label{fig:fragment}
\end{figure}

It is easy to see that for the case $w<b$ the SDA solution is optimal, that is, it coincides with the solution of \eqref{MP}. However, with a bit of work one can show that the SDA updates are suboptimal when $w>b$, even when taking advantage of \eqref{improvedupdate}. This example only shows there is room for improvement. We do not know how pervasive this type of instance is in practice, let alone how seriously it impacts learning.

We close this section with the observation that SDA is not that different from SGD applied to the loss function $x_c$. The main differences are that the step size (learning rate) is not a free parameter but set by deactivation events and that multiple steps may be executed when learning each data item. The benefit imparted by the rectified wire model, in this context, is just that the gradient computations and step sizes are very easy to compute. If we set aside monotonicity and solution guarantees, these attributes carry over to loss functions that are much stronger than $x_c$.

To get a sense of how small a change moves the model into difficult territory, consider the problem of learning the truth value of a Boolean function $f(z)$ on $N$ variables generated by $M$ binary \textsc{and}/\textsc{or} gates, as discussed in section \ref{sec:RWnets}. The rectified wire network would receive $2N$ doubled inputs and have output nodes $f$ and $\bar{f}$, trained so that $(x_f,x_{\bar{f}})=(f(z),\bar{f}(z))$. From corollary \eqref{cor:cor1} we know such a rectified wire network exists, with unit weights and having at most $7M$ hidden nodes. The outstanding question, of course, is whether a gradient based algorithm can find the appropriate bias settings.

Instead of the monotone loss function $x_{f(z)}$, fraught with the problem that training cannot prevent $(x_f,x_{\bar{f}})=(0,0)$ --- an ambiguity --- we might try the hinge-loss function
\[
\max\left(0,\;x_{f(z)}-x_{\bar{f}(z)}+\Delta\right).
\]
By corollary \eqref{cor:cor1} we know that it is possible to achieve zero loss over all $z$ with margin $\Delta=1$, and that the data will be correctly and unambiguously classified. Like $x_{f(z)}$, this loss function is piecewise linear on the rectified wire model and can likewise benefit from fast gradient and step size calculations. Because this loss is no longer monotone decreasing in the bias parameters, there is now no reason to initialize biases at zero. The loss of monotonicity, of course, brings with it the zero gradient problem. We suspect this may be more serious for rectified wires  than it is in standard network models.

\section{Clause learning}\label{sec:clause}

Although the SDA algorithm provides a tractable computation for learning individual data in the rectified wire model,
we do not yet have
any reason to believe this algorithm can learn any interesting functions. To address that concern, in this section we show that at least in a particular limit, the SDA algorithm is able to learn classes defined by any Boolean function. We already saw in section \ref{sec:RWnets} that rectified wire networks can represent such classes, even with favorable size scaling. To also demonstrate learning we have to rely on an admittedly impractical family of networks, with size growing exponentially with the number of Boolean variables.

\begin{defn}\label{def:CBN}
The complete Boolean network, for learning a Boolean function $f(z_1,\ldots,z_N)$, has $N$ pairs of input nodes, $2^N$ hidden nodes, and two output nodes. The input nodes correspond to the literals $z_1,\bar{z}_1,\ldots,z_N,\bar{z}_N$, the hidden nodes to all possible $N$-clauses $\{z_1,\bar{z}_1\}\times\cdots\times\{z_N,\bar{z}_N\}$ formed from $N$ literals, and the output nodes correspond to the truth value of $f$. Each hidden node has $N$ edges to each of its constituent literals, and edges to both of the output nodes.
\end{defn}
Figure \ref{fig:twonet} shows the complete Boolean network for functions of two variables (rendered for zero bias on all edges). In general, the hidden nodes all lie in the same layer. With the weights at the output nodes fixed at 1, the limit of the SDA algorithm we will analyze is where the weight $w$ shared by all the hidden nodes approaches zero. We refer to this limit on complete Boolean networks as \textit{clause learning}

The SDA bias changes in clause learning are very simple. While $(-\nabla x_c)_i\in \{0,1\}$ when $i$ is an output node, by \eqref{gradrecursion} we have $(-\nabla x_c)_i\in \{0,w\}$ at all the hidden nodes. From \eqref{gradrecursion} it also follows that the bias changes on edges into the hidden nodes are smaller by a factor $w$ relative to edges leaving the hidden nodes. Recall that an SDA iteration is defined by the deactivation of some edge as biases are increased. In clause learning this always occurs in the second layer, in the edges to the outputs. As the hidden node values are $O(w)$, the biases to the output edges will also be $O(w)$ as that suffices for deactivation. While there are corresponding bias changes on edges into the hidden nodes, these are smaller by a factor of $w$, or $O(w^2)$, and far from making these edges inactive (since $x_i\in \{0,1\}$ on the input nodes).

\begin{figure}[!t]
\begin{center}
\includegraphics[width=2.5in]{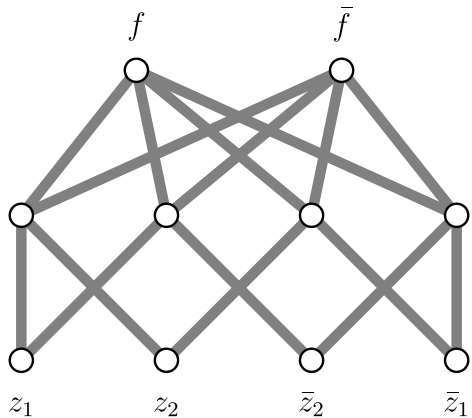}
\end{center}
\caption{Complete Boolean network for 2-variable Boolean functions $f(z_1,z_2)$. The function value is classified by the smaller of the two output node values, with the convention that node $f$ is smaller when $f(z_1,z_2)=0$.}
\label{fig:twonet}
\end{figure}

\begin{lem}\label{lem:CL}
In clause learning, the hidden nodes and the biases on edges leaving them, always have values of the form $w K+O(w^2)$, $K\in\mathbb{Z}_{\ge 0}$.
\end{lem}
\begin{proof}
We use induction, with the base case being the start of training where all the biases are zero. For any (doubled) input $d\in\{0,1\}^{2N}$, the value of a hidden node $h$ is $x_h=w K(d,h)$, where $K(d,h)\in\{0,\ldots,N\}$ counts the number of literals associated with node $h$ that are 1 for input $d$. Next suppose that a slightly modified statement continues to hold after $T$ iterations of SDA, that is, for any hidden node $h$ and any edge $h\to i$, $x_h,b_{h\to i} \in\{w K+O(w^2)\colon K\in\mathbb{Z}_{\ge 0}\}$. Consider what happens in iteration $T+1$, invoked when the network makes the wrong prediction and some edge is deactivated by an increase in bias. The deactivated edge will always be an edge leaving a hidden node, and the change of its bias will therefore be the difference of two numbers in the set $\{w K+O(w^2)\colon K\in\mathbb{Z}_{\ge 0}\}$, thus keeping the changed bias in this set. Since the corresponding bias changes on edges into the hidden nodes are $O(w^2)$, the statement about the values of the hidden nodes also continues to hold.
\end{proof}

\begin{thm}\label{thm:CL}
Clause learning succeeds for any Boolean function.
\end{thm}
\begin{proof}
Consider the values of the two output nodes when clause learning encounters the doubled input $d\in\{0,1\}^{2N}$ associated with the function argument $z\in\{0,1\}^{N}$. By lemma \ref{lem:CL}, $x_f=w K_f+O(w^2)$, $x_{\bar{f}}=w K_{\bar{f}}+O(w^2)$, where $K_f,K_{\bar{f}}\in\mathbb{Z}_{\ge 0}$. As long as the integer for the ``wrong" output is greater than zero, SDA will make bias changes, if necessary, to identify the correct class. For example, suppose $f(z)=0$ and $K_{\bar{f}}>0$. Through bias increases SDA can bring about $x_f<x_{\bar{f}}$ because all inactivations will be on edges $h\to f$, that is, from hidden nodes to the ``correct" output node. The only effect these changes can have on $x_{\bar{f}}$ is a decrease of order $O(w^2)$, due to $O(w^2)$ bias increases on edges into the hidden nodes. Since each edge $h\to f$ can become inactive just once, the number of SDA iterations required to produce the desired outcome is bounded by the number of hidden nodes.

The above argument does not apply in the case $K_{\bar{f}}=0$. However, we can argue that this never arises when $f(z)=0$. If it did, consider the input $d\in\{0,1\}^{2N}$ corresponding to $z$ and the unique hidden node $h$ whose combination of input-literals is such that $x_h=w N+O(w^2)$ (complete Boolean networks sample all $N$-literal combinations). Since $x_{\bar{f}}$ is a sum of non-negative contributions from all the hidden nodes, $K_{\bar{f}}=0$ is only possible if $b_{h\to \bar{f}}=w N+O(w^2)$. To finish the proof, we show that this value is inconsistent with conservative learning.

Fixing the hidden node $h$ from above, consider all the arguments $z$ for which $f(z)=1$ and the bias $b_{h\to \bar{f}}$ required to correctly assign them to the correct class by the smallness of $x_{\bar{f}}$. The corresponding inputs $d$ for such $z$ would produce $x_h=w K+O(w^2)$, with $K<N$, because the clause associated with $h$ is unique in producing $x_h=w N+O(w^2)$ only for an input for which $f(z)=0$. But to achieve the smallest contribution, $O(w^2)$, to $x_{\bar{f}}$ from edge $h\to \bar{f}$, we only need bias $b_{h\to \bar{f}}=w K+O(w^2)$, $K<N$. This is smaller (more conservative) than the hypothesized value from above.

The same argument, with interchanged output nodes, applies to learning data for which $f(z)=1$.
\end{proof}

\section{Balanced weights}\label{sec:weights}

We say a network is \textit{layered} when the nodes can be partitioned into a sequence of layers $\ell=0,1,\ldots,L$, such that all edges in the network are between nodes in adjacent layers. When a rectified wire network has the property that the weights are only layer-dependent, a suitable rescaling applied to the biases will eliminate all the weights --- replacing them by 1 --- without changing the classification behavior of the network. Denoting the layer of node $i$ as $\ell(i)$, we see from definition \ref{def:rectwire} that the rescalings
\begin{align}
\tilde{x}_{i}&=x_{i}/W_{\ell(i)}\nonumber\\
\tilde{y}_{i\to j}&=y_{i\to j}/W_{\ell(i)}\nonumber\\
\tilde{b}_{i\to j}&=b_{i\to j}/W_{\ell(i)}\label{biasrescale},
\end{align}
where
\[
W_{\ell(i)}=\prod_{\ell = 0}^{\ell(i)}w_\ell,
\]
leave the equations unchanged while replacing all the weights by
1. Here $\ell = 0$ and $\ell=L$ correspond to the input nodes and
output (class) nodes, respectively. Classification is unchanged
because all the output node values are rescaled by the same (positive) factor.

The weights do have an effect on the way the network is trained by the
conservative learning rule. In the following we motivate a particular
setting of the weights called \textit{balanced}, derived for each node $i$ from its in-degree, $|{\to} i|$, and out-degree, $|i{\to}|$.

From definition \ref{def:rectwire} we see that the choice
\begin{equation}\label{weight1}
w_i=\frac{1}{|{\to} i|}
\end{equation}
will have the effect that there is no net gain or decay in the typical node values $x$ when moving from one layer to the next. The appearance of the in-degree is associated with the forward propagation of $x$. 

A very different choice is suggested by recursion \eqref{gradrecursion}, which relates the bias changes $\dot{b}=(-\nabla x_c)$ throughout the network in conservative learning. 
If we wish the biases to have equal influence on the class output node, so that again there is no net gain/decay moving between layers, then the correct choice is 
\begin{equation}\label{weight2}
w_i=\frac{1}{|i{\to}|}.
\end{equation}
Here we have the out-degree because the bias changes are derived from backward propagation.

As edges become inactive over the course of training, the in-degree in
\eqref{weight1} and out-degree in \eqref{weight2} should count only
the active edges. However, both \eqref{weight1} and \eqref{weight2}
are problematic when they are unequal, even at the onset of edges
becoming inactive. When \eqref{weight1} and \eqref{weight2} are
unequal, the node values $x$ and accumulated bias changes $b$ will
grow at different rates from layer to layer. Since edge deactivation
is determined by $x-b$, it will not be uniform in the network, but
will be concentrated either at the input layer or the output layer.

To promote a more uniform distribution of edge activations in the network we use \textit{balanced weights}:
\begin{equation}\label{weight3}
w_i=\frac{1}{\sqrt{|{\to} i|\; |i{\to}|}},\quad i\in H.
\end{equation}
With this choice, the propagation of $x$ in the limit of small bias takes the form
\[
x_i=\frac{|{\to} i|}{\sqrt{|{\to} i|\; |i{\to}|}} \;x_{\to i},
\]
where $x_{\to i}$ denotes an average over nodes on the in-edges to $i$. From \eqref{gradrecursion} we have
\[
(\nabla x_c)_i=\frac{|i{\to}|}{\sqrt{|{\to} i|\; |i{\to}|}}\; (\nabla x_c)_{i\to},
\]
where $(\nabla x_c)_{i\to}$ denotes an average over nodes on out-edges of $i$. The two growth rates, adjusted for the same direction through the network, are now equal:
\[
\frac{x_i}{x_{\to i}}=\sqrt{\frac{|{\to} i|}{|i{\to}|}}=\frac{(\nabla x_c)_{i\to}}{(\nabla x_c)_i}.
\] 
Whereas neither $x$ nor $\dot{b}=(-\nabla x_c)$ will have zero layer-to-layer growth when there is a systematic imbalance in the in-degrees and out-degrees, by having equal growth there is a better chance that deactivations will occur uniformly across layers. Non-zero layer-to-layer growth/decay of variables in a rectified wire network does not present a problem because the equations have no intrinsic scale. Suitable rescalings of the type given at the start of this section, applied after training, can neutralize the layer-to-layer growth without changing the classification behavior of the network.

So far we have only considered the weights of the hidden nodes. The only other nodes with weights are the output nodes. These might be weighted differently based on prior information about the classes, such as their frequency in the data. However, when there is no distinguishing  prior information about the classes, the weights on the output nodes should be given equal values. By scale invariance we are free to impose
\[
w_i=1,\quad i\in C.
\]

As a tool for the study of rectified wire networks, we introduce a global weight-multiplier hyperparameter $q$ where formula \eqref{weight3} is replaced by
\begin{equation}\label{weight4}
w_i=\frac{q}{\sqrt{|{\to} i|\; |i{\to}|}},\quad i\in H,
\end{equation}
and $q=1$ corresponds to balanced weights. The limit $q\to 0$ is interesting because it collapses the rectified wire model to a much simpler one. As explained in section \ref{sec:clause}, and easily generalized to arbitrary numbers of layers, two things happen in the $q\to 0$ limit: (i) only edges in the final layer ever become inactive, and (ii) all biases except those in the final layer can be neglected. Because of (i), all the layers below the final hidden layer of nodes combine into a single linear layer, resulting in a model with just a single hidden layer. After scaling away the weights, the expression for the value of an output node $k\in C$ takes the form
\begin{equation}\label{q0limit}
x_k=\sum_{j\in H}\max{\left(0,\sum_{i\in D}a_{j i}\,x_i-b_{j\to k}\right)},
\end{equation}
where the integers $a_{j i}$ count the number of paths in the network from an input node $i$ to a hidden layer node $j$. This model, comprising a single layer of rectifier gates, is of course much easier to analyze than a general, multi-layered rectified wire model. By decreasing $q$ in experiments one can assess the value of depth in the network. In particular, if $q\to 0$ does not compromise performance, then depth is not being utilized in an essential way.

\section{Small networks}\label{sec:small}

Because both our network model (rectified wires) and training method (conservative learning) are unconventional, we use this section to demonstrate the model and method on very small networks before turning to more standard demonstrations in section \ref{sec:exp}. We first describe how the SDA algorithm trains the complete Boolean network on functions of two variables. This is followed by the study of a two-hidden-layer network, for functions on three Boolean variables.

We already proved (section \ref{sec:clause}) that the 16-edge network shown in Figure \ref{fig:twonet} learns all $2^{2^2}$ two-variable Boolean functions with SDA in the limit of vanishing weights in the hidden layer. Here we check that this continues to be true for balanced weights. Thanks to the symmetry of this network, we only need to check four equivalence classes of functions. Here equivalence is with respect to negating inputs or the output, or swapping the two variables. In all cases, these operations on the function correspond to relabelings of the input and output nodes of the network. Using the algebra of $\mathrm{GF}(2)$ for Boolean operations, representatives of the four equivalence classes are
\begin{align*}
f_0(z_1,z_2)&=0\\
f_1(z_1,z_2)&=z_1\\
f_+(z_1,z_2)&=z_1+z_2\\
f_\times(z_1,z_2)&=z_1 z_2.
\end{align*}

We have verified that the four functions above are learned on the 16-edge network with SDA and balanced weights ($q=1$). The biases were always initialized at zero, as rendered by equal-thickness wires in Figure \ref{fig:twonet}.
Recall how the network is trained on a function $f$ with the SDA algorithm. After a data vector $d=(z_1,z_2,\bar{z}_2,\bar{z}_1)$ is forward-propagated, the two output nodes will have equal or unequal values. If unequal, and the correct node --- corresponding to class $f(z_1,z_2)$ --- is smallest, no biases are changed and the next item is processed. If the smallest output is positive but wrong, or equal to the other output, then the SDA algorithm executes iterations until the correct output is smallest or zero. The network will then have learned $d$, except for the case where the iterations drive both outputs to zero. This irreversible mode of classification ambiguity does not arise with the 16-edge network and any of the four functions.

In ultra-conservative learning there is no guarantee that successfully learned data is not unlearned in subsequent training. However, by testing on  all four data each time biases are changed by SDA, one can establish whether the function is learned. The order in which data are processed matters, as manifested in the final bias values. Figure \ref{fig:fourfunc} shows two examples of final bias settings for each of the four functions.

\begin{figure}[!t]
\begin{center}
\includegraphics[width=5.in]{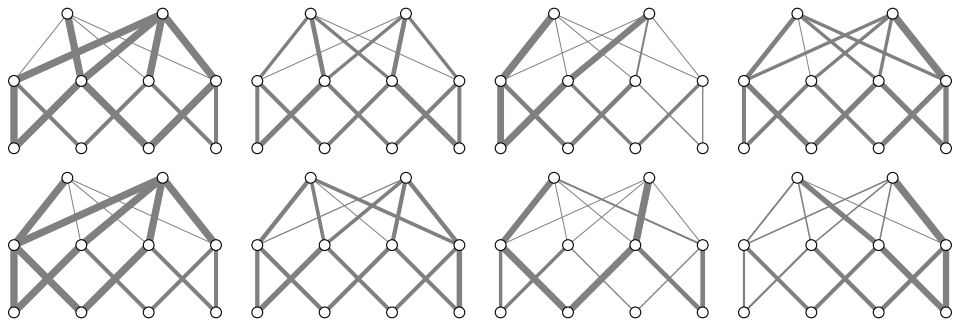}
\end{center}
\caption{Two examples of final bias settings, in columns, for networks trained on the functions (left to right) $f_0$, $f_1$, $f_+$, $f_\times$ with different data orderings.}
\label{fig:fourfunc}
\end{figure}

Define a \textit{trial} as a single training experiment where, starting with zero bias, the data are processed in some prescribed order. All trials, for any of the four functions above, are successful on the 16-edge network. That is, there always is a point where the biases stop changing and the data are correctly and unambiguously classified. Two numbers of interest in a successful trial are (i) $\mathrm{\#}_\mathrm{err}$, the total number of wrong or ambiguously classified data (invocations of SDA) over the course of training, and (ii) $\mathrm{\#}_\mathrm{iter}$, the total number of SDA iterations performed. These numbers are equal on the 16-edge network trained on any of the 2-variable functions, as one SDA iteration is sufficient to fix any incorrect classification in this case. Not surprisingly, the constant function $f_0$ is the easiest to learn, with $\mathrm{\#}_\mathrm{err}=1$. All trials with $f_1$ have $\mathrm{\#}_\mathrm{err}$ equal to 4 or 5. The \textsc{xor} function, $f_+$, always has $\mathrm{\#}_\mathrm{err}=3$, while the \textsc{and} function, $f_\times$, has the greatest variation with $\mathrm{\#}_\mathrm{err}=2,3$ or $4$.

The more aggressive variant of SDA, where iterations continue until the class output node is zero, also succeeds for all four functions on the 16-edge network with balanced weights. In this mode of training one pass through the data always suffices. We find, for all four functions, that the network must see all the data ($\mathrm{\#}_\mathrm{err}=4$) and $6\le\mathrm{\#}_\mathrm{iter}\le 9$.

Complete Boolean networks, such as the 16-edge network, quickly become too large for functions of many Boolean variables, and alternative network designs must be considered. For example, the nodes in the single hidden layer could be clauses formed from all small subsets of the variables, or even an incomplete sampling of such clauses. The success of clause learning ($w\to 0$) with such networks would of course depend on the nature of the function.   

The opposite limit applied to the hidden-layer weights, $w\to \infty$,
suggests a different design for single-hidden-layer networks. Because
biases on edges to the outputs never change in this limit, we connect
each hidden-layer node to a single output. The hidden-layer nodes are
now interpreted as proto-clauses, because their composition in
literals is modified during training by changes to the input
biases. For example, a proto-clause could sample both a variable and
its negation, and rely on training (bias change) to select one or the
other. Moving bias changes to the input side of the network
alleviates the exponential explosion of static clauses one faces in
the $w\to 0$ limit.

The departure from complete Boolean networks we feel is most
interesting is increasing the number of hidden node layers. Indeed, keeping training tractable in this setting was the primary motivation for the rectified wire model. To explore this feature we trained a network on three-variable Boolean functions where the hidden layer nodes only have in-degree two. To potentially express relationships among three variables, the hidden nodes are arranged in two layers.

The three-variable network, shown in Figure \ref{fig:threenet}, is
complete in the following sense. The nodes in the first hidden layer
exhaust the $3\times 4$ ways of collecting inputs from distinct
Boolean pairs, with and without negation. One attribute of these nodes
is the identity of the Boolean variable --- 1, 2 or 3 ---  that was
\textit{not} sampled. The second layer of hidden nodes exhausts all $3\times 4\times 4$ combinations of input pairs for which this first hidden layer attribute is different. As in the 16-edge network, the wiring between the final hidden layer and the output nodes is complete in the conventional sense. The resulting network has 216 edges.

\begin{figure}[!t]
\begin{center}
\includegraphics[width=4.in]{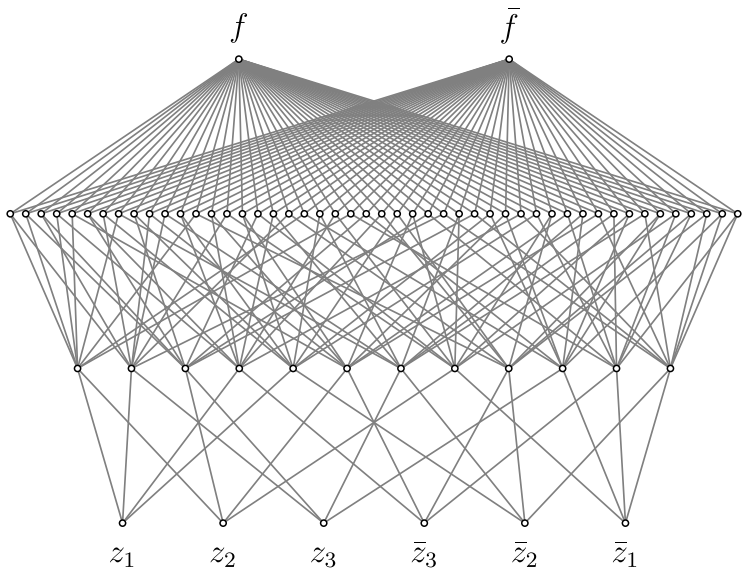}
\end{center}
\caption{A rectified wire network with two hidden layers that, we conjecture, learns all Boolean functions $f(z_1,z_2,z_3)$ on three variables. All hidden nodes in this network have two inputs.}
\label{fig:threenet}
\end{figure}

Bias renderings, after training, of networks with over 200 edges are not very comprehensible. However, the statistics of the numbers $\mathrm{\#}_\mathrm{err}$ and $\mathrm{\#}_\mathrm{iter}$ appear to correlate well with the naive complexity of the Boolean function. We considered three functions that evaluate to 0 and 1 with equal frequency:
\begin{align*}
f_1(z_1,z_2,z_3)&=z_1\\
f_>(z_1,z_2,z_3)&=z_1 z_2+z_2 z_3+z_3 z_1\\
f_+(z_1,z_2,z_3)&=z_1+ z_2+ z_3.
\end{align*}
Learning $f_1$, or learning to ignore $z_2$ and $z_3$, should be easiest. Function $f_>$, the logical majority function, we expect to be harder because its value, in some but not all cases, is sensitive to all three variables. By the same argument, the parity function $f_+$ should be the hardest of the three.

All trials we performed on $f_1$, $f_>$ and $f_+$, using SDA on the 216-edge network with balanced weights, were successful. Figure \ref{fig:threenetstats} shows the distribution of $\mathrm{\#}_\mathrm{err}$ and $\mathrm{\#}_\mathrm{iter}$ in 50 trials for each, differing only in the order in which the data are processed. By either metric --- the total number of false classifications or the total work in training --- the difficulty ranking of the functions is $f_1\;<\;f_>\;<\;f_+$. Less extensive trials on the full set of $2^{2^3}$ Boolean functions all proved  successful and leads us to conjecture that the 216-edge network of Figure \ref{fig:threenet} learns all 3-variable Boolean functions regardless of data order.

\begin{figure}[!t]
\begin{center}
\includegraphics[width=4.in]{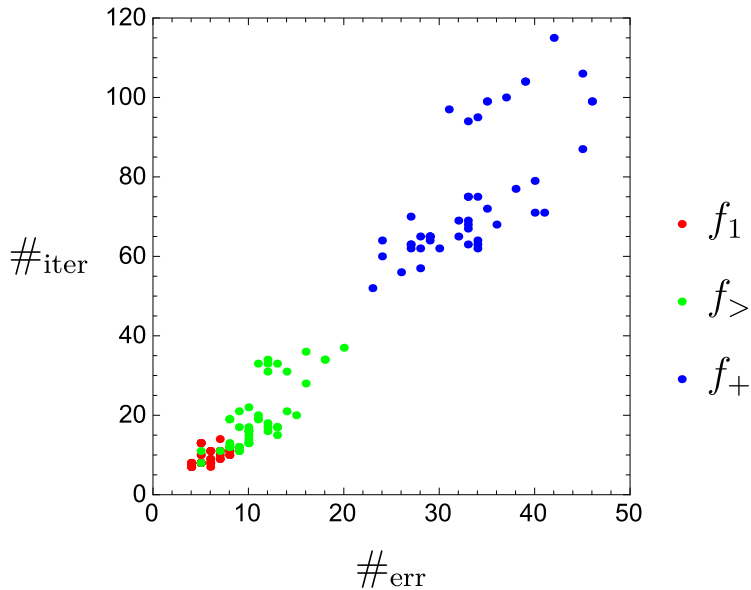}
\end{center}
\caption{Distribution of $\mathrm{\#}_\mathrm{err}$ and $\mathrm{\#}_\mathrm{iter}$ in 50 trials when learning the three-variable Boolean functions $f_1$, $f_>$ and $f_+$ on the network shown in Figure \ref{fig:threenet}.}
\label{fig:threenetstats}
\end{figure}

We note that the complete Boolean network, with only $2^3(3+2)=40$ edges, would have been much smaller in this instance and by theorem \ref{thm:CL} can learn all the Boolean functions in the $w\to 0$ limit. Still, it is reassuring to see that SDA also works in deep networks.

The three-variable Boolean functions also highlight the advantage of the ultra-conservative termination rule for SDA iterations. When we repeat the 50 trials with the stricter zero-output-value rule ($x_c=0$), we find that training is not always successful. While training on $f_1$ was always successful, the success rate on $f_>$ was 76\% and on $f_+$ only 22\%. The success rate improves for smaller $q$ and appears to be at 100\% for all three functions when $q=0.5$. Another point in favor of ultra-conservative termination is the greatly reduced number of iterations. For example, by forcing zero outputs, the total number of iterations needed to learn $f_1$, averaged over data ordering, is about 140, or about 15 times the number needed by the ultra-conservative rule.

\section{Sparse expander networks}\label{sec:expander}

A nice simplification provided by rectified wire networks is that the training algorithm has very few hyperparameters. Learning succeeds or fails mostly on the basis of the network we choose. However, it is possible that the tractability of rectified wire training comes at the cost of a greater sensitivity to network design. We therefore anticipate that network characteristics --- size, depth, sparsity --- will take over the role of hyperparameters.

As a general rule, learning improves with larger networks. Since the edge-biases are the only learned parameters of a rectified wire network, the appropriate measure of network size is the number of edges. In this section we introduce a family of networks that adds depth --- the number of hidden layers --- as a second design parameter.

Layered networks that are completely connected in the conventional sense suffer from a symmetry breaking problem. Since we start with all biases set at zero, a completely connected network has perfect permutation symmetry of its hidden nodes. This symmetry extends, during training, to the bias updates on all edges between the hidden nodes. To allow for independent settings of the biases, the initial biases must break the symmetry in some data-neutral fashion. Incompletely connected or even sparse networks are an extreme form of symmetry breaking, and correspond to infinite bias settings.

Our network design, called \textit{sparse expander networks}, generalizes the networks of section \ref{sec:small}. Sparsity in our case is the property that all the hidden nodes have in-degree of only two. At the same time, the representation of the data is expanded, layer to layer, by a constant growth factor $g$ in the size of layers. The expander structure cannot be applied to the final two layers because of the fixed (and small) number of output nodes. Here the layers are completely connected, thereby making all the information in the last (largest) hidden layer available to each output node. Letting $h$ denote the number of hidden layers, the number of edges in a sparse expander network for $|D|$ inputs and $|C|$ classes, with parameters $(g,h)$, is
\begin{equation}\label{edgecount}
|E|=|D|(2g+\cdots+2g^{h-1}+(2+|C|)g^h).
\end{equation}

To introduce a degree of uniformity, our networks have a constant out-degree of $2 g$ on all input and hidden nodes. This requires that $g$ is an integer. The in-edges to the nodes in a hidden layer are generated by the simplest random algorithm. In a first pass through the hidden layer we create one edge per node to the layer below by drawing uniformly from the nodes that currently have the smaller of two possible out-degrees. A second edge is added in a second pass, at the completion of which all the nodes on the layer below will have out-degree $2g$. Edges from the first hidden layer to the input layer are constructed no differently. This construction is implemented by the publicly available\footnotemark C program \textbf{\texttt{expander}} and was used for all the experiments reported in the next section.

We can use the hyperparameter $q$ to assess the role of depth. If performance does not degrade in the limit $q\to 0$, then the network could have been replaced by the much simpler single-hidden-layer model \eqref{q0limit}. When this limit is applied to a sparse expander network, the number of hidden nodes is $|H|=|D|g^h$. The sampling of the fixed weights $a_{j i}$ in the rectifier gates is also controlled by $g$ and $h$.

\section{Experiments}\label{sec:exp}

To describe experiments with conservatively trained rectified wire networks it suffices to specify the network. The SDA algorithm has no other hyperparameters, since there even is a balancing principle that sets the weights. The two-parameter sparse expander networks offer a convenient way to study behavior both with respect to network size and depth. These characteristics were our main focus and guided the choice of experiments going into this study.
Follow-up experiments, that featured the weight-multiplier hyperparameter $q$, later provided critical insights for network design that have yet to be explored.

All our experiments were carried out with a publicly available\footnotemark[\value{footnote}] C implementation of the SDA algorithm called \textbf{\texttt{rainman}}. This program  requires that the user-supplied network is layered and the number of inputs and outputs are compatible with the data. Our networks were all created by the program \textbf{\texttt{expander}}.

\footnotetext{\texttt{github.com/veitelser/rectified-wires}}

Although \textbf{\texttt{rainman}} reports results at regular intervals, after a specified number of data have been processed, this ``batch size" has no bearing on the actual training because bias parameters are only updated in response to wrong classifications. The batch size merely sets the frequency with which the network is tested against a body of test-data. The main result of this test is the classification accuracy: the test-data average of the indicator variable that is 1 for correct, unambiguous classifications, and $1/m$ when the correct class node is in an $m$-fold tie.

In addition to the classification accuracy, \textbf{\texttt{rainman}} also reports a number of other quantities of interest. Two we have already seen in section \ref{sec:small}: the running total of misclassifications, $\mathrm{\#}_\mathrm{err}$, and the work (iterations) performed so far by SDA, $\mathrm{\#}_\mathrm{iter}$. These should increase while the classification accuracy is below 100\%. However, in an ``overfitting" situation $\mathrm{\#}_\mathrm{err}$ and $\mathrm{\#}_\mathrm{iter}$ stop increasing before the accuracy has reached 100\%. Evidence of this phenomenon is more directly discerned through another statistic reported by \textbf{\texttt{rainman}}: the average number of SDA iterations required to fix each of the incorrect classifications encountered in the batch, $\langle\mathrm{\#}_\mathrm{iter}\rangle$. As defined, $\langle\mathrm{\#}_\mathrm{iter}\rangle$ is never less than 1. However, when there are no misclassifications in the batch at all (keeping $\mathrm{\#}_\mathrm{err}$ and $\mathrm{\#}_\mathrm{iter}$ static), the value $\langle\mathrm{\#}_\mathrm{iter}\rangle=0$ is reported.

Information more revealing of the internal workings of the network is provided by the layer-averages of the edge activations, $\langle\alpha_1\rangle,\ldots,\langle\alpha_h\rangle$. Here $\langle\alpha_h\rangle$ is the fraction of edges, in the layer leading to the output nodes, that are active in an average test-data item. These numbers decrease during training, a result of biases being increased.

Lastly, a statistic relevant to the termination of training is the frequency, in the test-data, of irreversible misclassifications, $\langle 0_\mathrm{err}\rangle$. The latter arise whenever an output node, not of the correct class, has value zero. The onset of $\langle 0_\mathrm{err}\rangle>0$ signals that training should terminate because activation levels are so low that SDA is forced to make multiple outputs zero even while targeting just the class output. A simple remedy for preventing $\langle 0_\mathrm{err}\rangle>0$ is to increase the size of the network. The case of $m$-fold classification ambiguities, mentioned above, occurs in practice only when $m$ output nodes are tied at zero. Ambiguous classification is therefore not a problem as long as $\langle 0_\mathrm{err}\rangle>0$.

On a single Intel Xeon 2.00GHz core \textbf{\texttt{rainman}} runs at a rate of 50ns per iteration per network edge. Multiplying this number by the product of $\mathrm{\#}_\mathrm{iter}$ and  $|E|$ \eqref{edgecount} gives a good estimate of the wall-clock time in all our experiments.

\subsection{Nested majority functions}\label{sec:nmf}

We first describe experiments with a synthetic data set. Synthetic data has some advantages over natural data. Our \textit{nested majority functions} (NMF) data has two nice features. First, through the level of nesting we can systematically control the complexity of the data. Second, because these Boolean functions are defined for all possible Boolean arguments, it is easy to generate large training sets for supervised learning.

NMFs are a parameterized family of Boolean functions that evaluate to 0 and 1 with equal frequency. The first parameter, a prime $p$, sets the number of arguments at $p-1$. Instances are defined by three integers $a,b,c\in\{1,\ldots,p-1\}$ and a nesting level $n=0,1,2,\ldots$ . Starting with
\[
f^0_a(z_1,\ldots,z_{p-1})=z_a,
\]
the other functions (taking values in $\mathrm{GF}(2)$) are defined via the recurrence
\[
f^n_a=f_>\left(f^{n-1}_{(a b\Mod{p})}+z(a c)\;,\;f^{n-1}_{(2a b\Mod{p})}+z(2a c)\;,\;f^{n-1}_{(3a b\Mod{p})}+z(3a c)\right),
\]
where $f_>$ is the 3-argument majority function and 
\[
z(x)=\left(x\Mod{p}\right)\Mod{2}
\]
should be interpreted as an element of $\mathrm{GF}(2)$. The functions $f^n$ at level $n$ correspond to depth-$n$ Boolean circuits taking $p-1$ inputs and constructed from \textsc{not} and \textsc{majority} gates. Figure \ref{fig:majfunc} shows an instance of the function $f^5$ for $p=7$. 

\begin{figure}[!t]
\begin{center}
\includegraphics[width=3in]{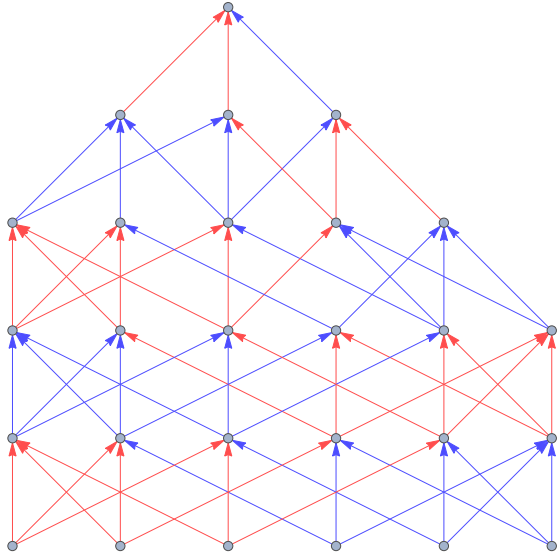}
\end{center}
\caption{An instance of a depth 5 nested majority function on 6 variables. The nodes are majority gates and the two edge colors indicate the presence/absence of negation. The depth $4,\ldots,1$ functions correspond to removing $1,\ldots,4$ layers from the bottom of the circuit.}
\label{fig:majfunc}
\end{figure}

From the identity
\[
f_>(z_1+1,z_2+1,z_3+1)=f_>(z_1,z_2,z_3)+1
\]
it follows that negating all the arguments of a NMF negates its value. The two classes, defined by the NMF value, are therefore equal in size.

Our experiments used data generated by the functions $f^n_1$ for $p=31$ and $b=2$, $c=3$. The easiest of these, with $n=0$, only depends on a single argument. To learn this function the rectified wire network (taking $2\times 30$ inputs) must identify the correct one and ignore all the others. In the case of $f^1_1$ it must select three of the arguments and also whether to apply negations. The functions $f^2_1$, $f^3_1$, etc. become progressively harder as they depend on more arguments and with more complicated rules. The possibility of generalizing the value of $f^n$, from many fewer than $2^{p-1}$ data, is clearly only possible for small $n$. Our results go as far as $n=4$.

We generated training and test data samples (of the 30 Boolean variables) with a pseudo-random number generator. In the longest experiment, with $10^7$ training samples, less than 1\% of the data was seen more than once by the training algorithm.

We first present results for the function $f^3_1$. Fixing the number of hidden layers at $h=3$, Figure \ref{fig:majcompareh3} shows the effect of increasing the network size via the growth factor $g$. The accuracy is plotted as a function of $\mathrm{\#}_\mathrm{err}$, rather than the number of training data, because learning (change of bias) occurs only when there has been a misclassification. In all except the $g=14$ network (683760 edges), the accuracy reaches a maximum after which there is a sharp decline that coincides with the onset of $\langle 0_\mathrm{err}\rangle>0$. The smaller networks could still serve as imperfect classifiers by terminating their training at the empirical maximum test accuracy, or, in the absence of test data, when $\langle 0_\mathrm{err}\rangle$ exceeds a threshold. We have no explanation of the intriguing similarity of the unsuccessful accuracy plots. We also find that many of the seemingly random features in the accuracy plots are nearly reproduced in different random instances of expander networks with the same $(h,g)$. 

We emphasize that it is not really surprising that generalization performance can drastically deteriorate with more training examples for small networks. Conservative learning is using a fundamentally different principle than the standard approach that tries to minimize an overall loss function across training examples. Conservative learning is merely trying to myopically account for the most recently seen example, which can have negative consequences for other training examples, to say nothing of unseen test examples.

\begin{figure}[!t]
\begin{center}
\includegraphics[width=5.in]{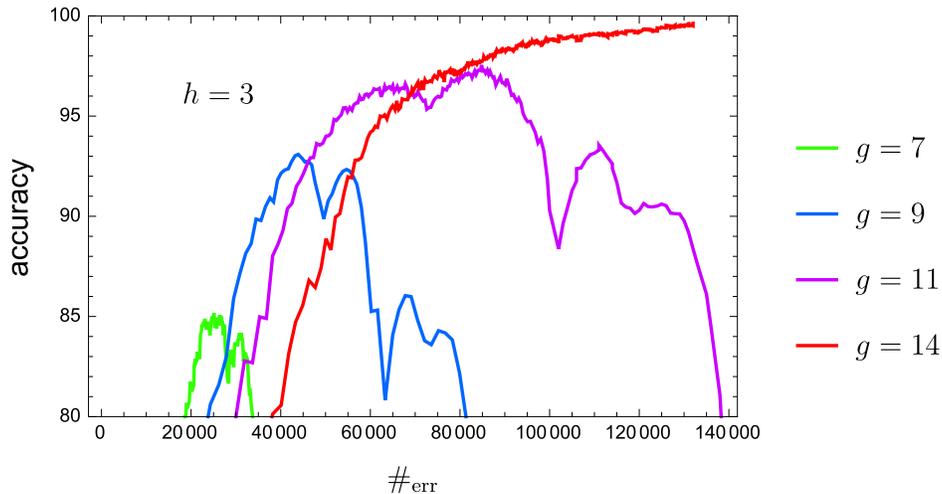}
\end{center}
\caption{Improvement of the classification accuracy with network size, for the function $f^3_1$, on rectified wire networks with three hidden layers.}
\label{fig:majcompareh3}
\end{figure}

The effect of changing network depth is shown in Figures \ref{fig:majcomparedepth1} and \ref{fig:majcomparedepth2}. Although the three featured networks have about the same number of edges, only the two deeper networks are successful. Figure \ref{fig:majcomparedepth1}, with accuracy plotted against $\mathrm{\#}_\mathrm{err}$, shows that learning is faster on shallower networks. On the other hand, when plotted against $\mathrm{\#}_\mathrm{iter}$ in Figure \ref{fig:majcomparedepth2}, we see that less work is required to train the deeper networks. Work, as measured in SDA iterations, is not evenly distributed over the course of training. Figure \ref{fig:majaveiter} shows $\langle\mathrm{\#}_\mathrm{iter}\rangle$ plotted against $\mathrm{\#}_\mathrm{err}$. More iterations are needed early in training, especially so in shallow networks.

The balanced weight heuristic is far from successful in keeping the layer-averaged activation levels uniform. After training the $(h,g)=(3,14)$ network, we find $\langle\alpha_3\rangle=1.1\%$, $\langle\alpha_2\rangle=47\%$ with the rest at 100\%. For the deeper and also successful $(h,g)=(4,7)$ network, the diminished activations are $\langle\alpha_4\rangle=4.1\%$, $\langle\alpha_3\rangle=57\%$, $\langle\alpha_2\rangle=76\%$.

The lower complexity of the functions $f^n_1$ for $n<3$ is reflected in the smaller networks required to learn them. On networks with two hidden layers ($h=2$), $f^2_1$ is successfully learned with $g=21$ (108360 edges) after only 5000 misclassifications. For $f^1_1$ the corresponding numbers are $g=6$ (9360 edges) and $\mathrm{\#}_\mathrm{err}=900$. On the other hand, if we train $f^4_1$ on the same network ($h=4$, $g=7$) that learned $f^3_1$ with $\mathrm{\#}_\mathrm{err}\approx 4\times 10^5$, the best accuracy, achieved for roughly the same $\mathrm{\#}_\mathrm{err}$, is only 91.9\%.

\begin{figure}[!t]
\begin{center}
\includegraphics[width=5.in]{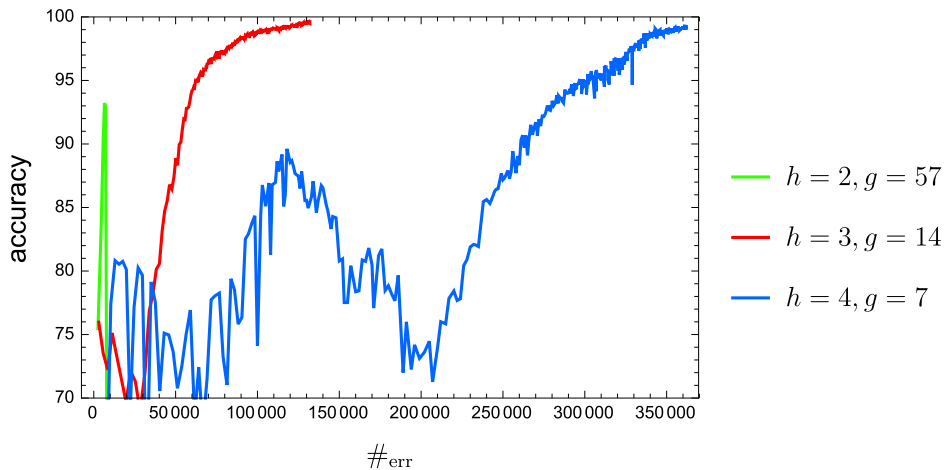}
\end{center}
\caption{Of three networks with comparable size trained on the function $f^3_1$, the deeper networks are more successful and require more data.}
\label{fig:majcomparedepth1}
\end{figure}

\begin{figure}[!t]
\begin{center}
\includegraphics[width=5.in]{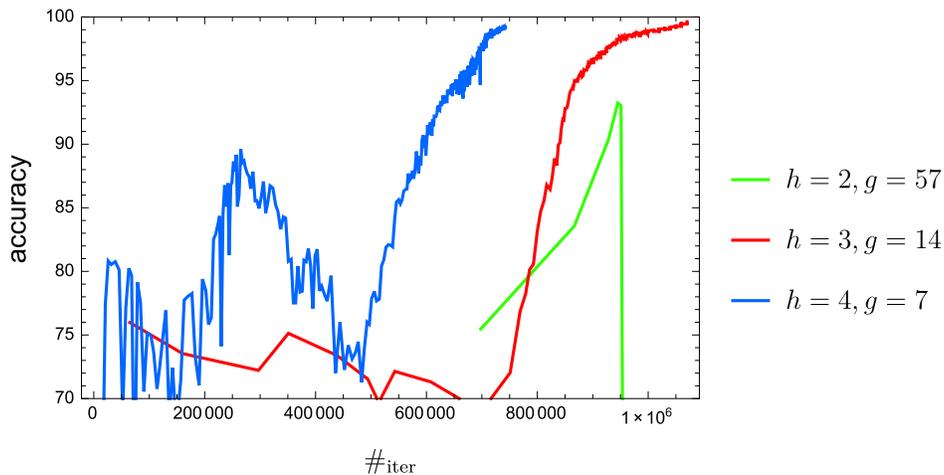}
\end{center}
\caption{Same as Figure \ref{fig:majcomparedepth1}, except plotted against the total work (SDA iterations).}
\label{fig:majcomparedepth2}
\end{figure}

\begin{figure}[!t]
\begin{center}
\includegraphics[width=5.in]{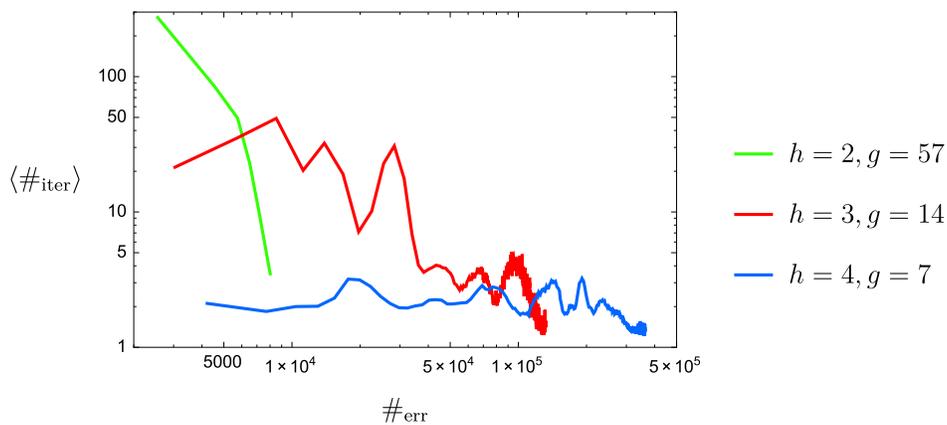}
\end{center}
\caption{Average number of iterations used by the SDA algorithm to fix a wrong classification over the course of training for the three networks featured in Figs. \ref{fig:majcomparedepth1} and \ref{fig:majcomparedepth2}.}
\label{fig:majaveiter}
\end{figure}

\begin{figure}[t]
\begin{center}
\includegraphics[width=5.in]{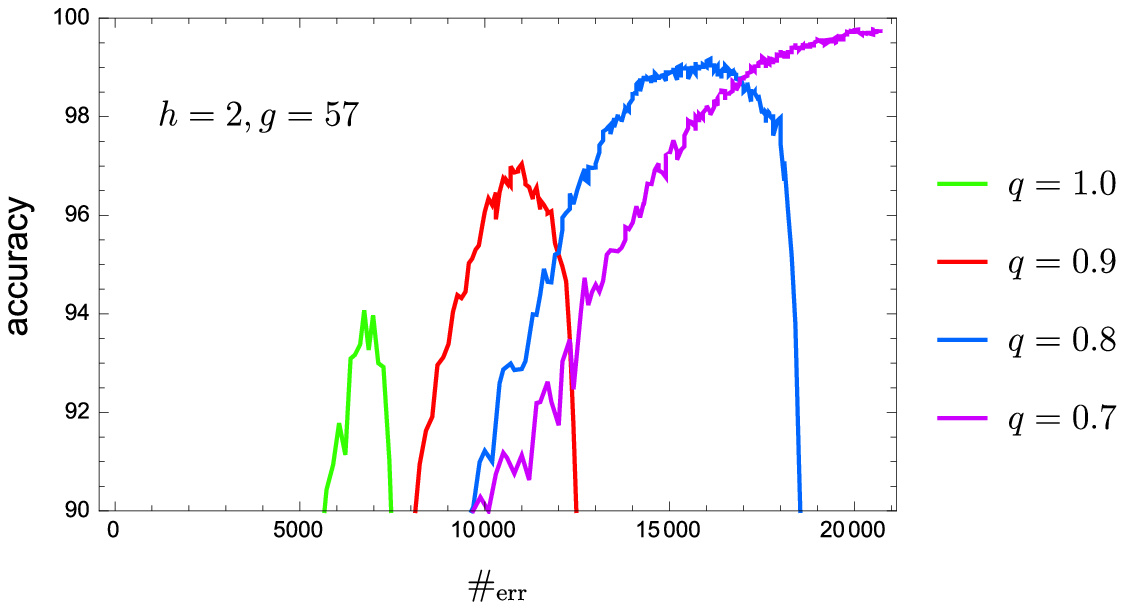}
\end{center}
\caption{Evolution of the prediction accuracy with the weight-multiplier hyperparameter $q$ on the $h=2, g=57$ network for the function $f^3_1$.}
\label{fig:qseries}
\end{figure}

The prediction accuracy on NMF data, for all of our networks, is improved by decreasing the weight-multiplier hyperparameter $q$. In fact, it appears that $q\to 0$ is optimal, or that the best performance can be achieved with the model \eqref{q0limit} comprising a single layer of rectifier gates with fixed weights $a_{j i}$. A striking case is the function $f^3_1$ trained on the sparse expander graph with $h=2, g=57$, already encountered in Figure \ref{fig:majcomparedepth1}. Figure \ref{fig:qseries} shows the effect of decreasing $q$ from the balanced setting, $q=1$. We see that the 94\% accuracy of the $q=1$ network reaches 100\% already for $q=0.7$. It seems the only downside in decreasing $q$ is an increase in the number of training data (greater $\mathrm{\#}_\mathrm{err}$). The $q\to 0$ weight ``codewords" $a_{j i}$ associated with the sparse expander network are very simple. Because the network has in-degree two and depth $h=2$, almost all the rectifiers see four random inputs that are given equal weight. It is likely that the performance of model \eqref{q0limit} can be further improved with weight codewords that go beyond those generated by sparse expander networks.

In light of the $q\to 0$ findings, it is not surprising that conventional neural networks, even with just a single hidden layer, easily outperform the rectified wire model on the NMF data. A fully-connected ReLU network with one layer of only 50 hidden nodes was sufficient to learn $f^3_1$ completely after seeing $5\times 10^5$ samples. This can be compared with the $2\times 10^4$ samples ($\mathrm{\#}_\mathrm{err}$) that informed bias changes in the rectified wire model when learning the same data (Figure \ref{fig:qseries}). The vast reduction in the number of hidden nodes, from $2\times 10^5$ (penultimate layer of the $h=2, g=57$ network) to 50, no doubt is attributable to the standard model's capacity for ``designing" the weight-codewords that in model (16) are static.

\subsection{MNIST}\label{sec:mnist}

Seen as images, the MNIST handwritten digits \citep{lecun1998gradient} are analog data. Because the contrast in these images is well approximated as two-valued, it makes sense to map pixels to a two-symbol alphabet and train a rectified wire network on data of that type. We report results with this approach at the end of this section. However, we first try processing MNIST non-symbolically, using the analog data vectors described in equation \eqref{analogvec}.

One motivation for the analog approach is that MNIST images are highly compressible. Playing to a possible strength of rectified wire networks, it might be advantageous to first map MNIST images, and images in general, to data vectors the components of which are more nearly statistically independent. We use a method based on principal component analysis.

Using singular value decomposition, we approximate the $60000\times 28^2$ matrix $X$ of flattened MNIST training images as the product $X\approx U V$, where the rows of the $r\times 28^2$ matrix $V$ are the orthonormal ``eigen-images" associated with the $r$ largest singular values of $X$. A flattened image $x$ (training or testing) is mapped to $r$ analog channels by $v=V x$. Applying the component-wise cumulative probability distribution functions to this $v$ then gives us the $2r$-component data vector as in equation \eqref{analogvec}. We compute the cumulative probability function from the training data and use the same function when processing the test data (with test samples below the minimum or above the maximum training samples mapped to 0 and 1 respectively).

Figure \ref{fig:mnistanalogacc} shows the test accuracy when training a sparse expander network on analog MNIST data with $r=50$ channels. This network has 205400 edges, the same number a conventional network would have with a single fully-connected layer of 259 hidden nodes. After making about $\mathrm{\#}_\mathrm{err}=25000$ misclassifications, there are no further misclassifications on the training data and training ceases. The test accuracy attains a maximum of $94.54\%$ at this point. Figure \ref{fig:mnistanalogiter} shows the evolution of $\langle\mathrm{\#}_\mathrm{iter}\rangle$ over the course of training. Generalization (test accuracy) improves when we decrease the $q$ hyperparameter from the value $q=1$ for balanced weights. However, the improvement, reaching $94.99\%$ at $q=0.2$, is much more modest than we found for the NMF data. The time used by \textbf{\texttt{rainman}} in any of these training runs was about 65 minutes. Although about one million data (16 epochs) are processed, only for about $25000$ of these does SDA do anything.

We also used the analog MNIST data to compare SDA training with stochastic gradient descent (SGD) training on the same model and network architecture. Because the sparse networks we use for SDA are a poor fit for GPU-based software, we wrote a C++ SGD optimizer that, like SDA, runs on a single core without calls to a GPU. In the strictest comparison SGD used the same monotone loss function, $x_c$, used by SDA. For SGD we also tried multi-class hinge loss, with margin parameter $0.1$. As discussed in section \ref{sec:class}, this loss spoils tractability (monotone bias changes) and is therefore not applicable for SDA. The mini-batch size for SGD was fixed at 100 and we employed standard stochastic gradient descent without momentum.

SGD does poorly when forced to use SDA's monotone loss. With the learning rate set at $0.002$, training accuracy reaches a maximum of $91.5\%$ in 6 minutes; this is also the test accuracy for this mode of operation. Since SDA uses adaptive gradient steps (set by deactivations), we also tried SGD with a ten-fold reduction in learning rate. While training now goes through 97 epochs of data and takes about as long as SDA, there is no improvement in either the training or test accuracy. We interpret this as supporting the conservative learning principle: there are gains when parameters are optimally adjusted in response to individual data.

The performance of SGD improves significantly when the rectified wire model is trained with multi-class hinge loss. However, due to the structure of expander networks, this loss is problematic when applied exclusively. When evaluated on the initial expander network (all edges active), the partial derivatives of the loss with respect to all biases, except those on edges to the outputs, are identically zero. This occurs because 
the effect of all bias changes (except in the final layer) passes through the penultimate layer of nodes that is completely connected to the outputs, and equal changes on these leaves the hinge loss unchanged. To avoid the zero derivative problem we used a hybrid approach, where hinge loss was used only after some number of data had first been processed with the monotone $x_c$ loss to create deactivations that break the symmetry in the final layer. Switching to hinge loss after 5 epochs had no effect on accuracy, but switching after 10 epochs improved training and test accuracies to $94.4\%$ and $94.1\%$, respectively. This trend in improvement continues and reaches $96.8\%$ (training) and $95.6\%$ (test) when the switch is made after 20 epochs (121 minutes). We attribute the modest improvement in generalization ($0.6\%$), over the SDA trained model, to the margins that hinge loss is able to insert in the class boundaries.

\begin{figure}[!t]
\begin{center}
\includegraphics[width=4.in]{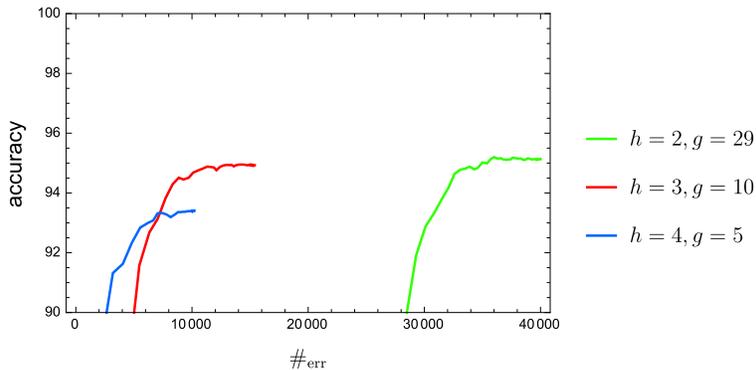}
\end{center}
\caption{Classification accuracy of MNIST images from 50-principal-component data vectors with the SDA algorithm on a sparse expander network with parameters $h=2$, $g=13$.}
\label{fig:mnistanalogacc}
\end{figure}

\begin{figure}[!t]
\begin{center}
\includegraphics[width=4.in]{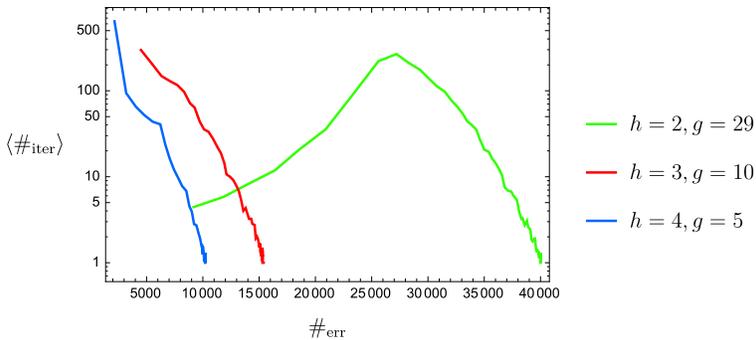}
\end{center}
\caption{SDA iterations per misclassification for the same data and network as in Figure \ref{fig:mnistanalogacc}}.
\label{fig:mnistanalogiter}
\end{figure}

We do not believe the overfitting we saw with SDA was caused by limiting the analog data vectors to only 50 channels, because essentially the same classification accuracy is obtained with the uncompressed symbolic data. The latter was generated from the raw MNIST images by thresholding pixel values at $40\%$ of the maximum (the minimum point of the pixel histogram) to define the two symbols. Figure \ref{fig:mnistbinacc} shows the evolution of the accuracy at fixed depth $h=4$ as the network size increases. In the smallest network ($g=1$), as expected, the accuracy drops precipitously at the onset of $\langle 0_\mathrm{err}\rangle>0$. However, already in the next network ($g=2$) we have the situation where training terminates because most (training) data are correctly classified and the remaining ambiguous cases are ties at zero (in output) and untrainable. At termination we find $\langle 0_\mathrm{err}\rangle=2.7\%$, contributing to the reduction in accuracy. Increasing the network size further ($g=3$) gives accuracy 94.8\% at termination, where now $\langle 0_\mathrm{err}\rangle=0$. This accuracy is close to that obtained with the analog data.

Although our results for MNIST are an improvement over the simple linear classifier with accuracy 88\%, and even a pairwise-boosted linear classifier (92.4\%), the appropriate comparison should be with other tractable methods. Of these, we note that support vector machines with a Gaussian kernel already achieve 98.6\% accuracy on MNIST. These and other techniques for MNIST are reviewed in~\cite{lecun1998gradient}.

\begin{figure}[!t]
\begin{center}
\includegraphics[width=5.in]{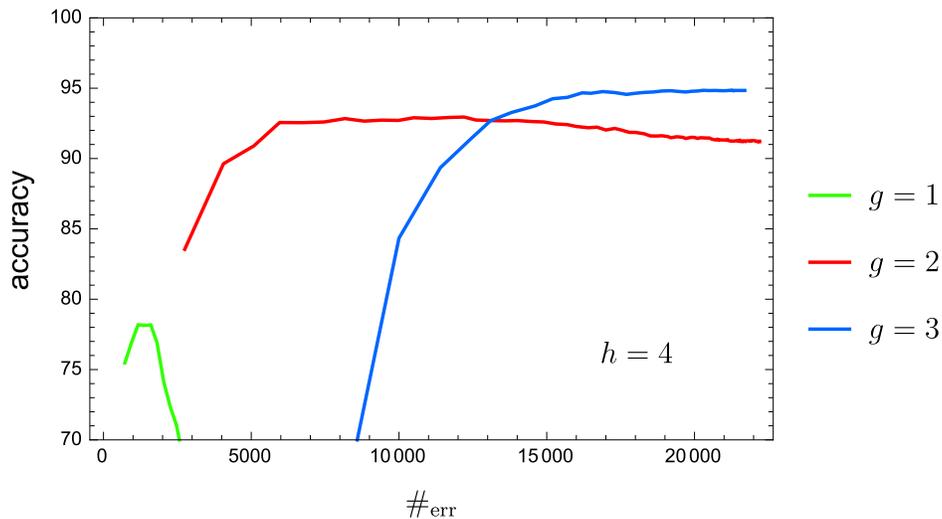}
\end{center}
\caption{Classification accuracy of binarized MNIST images on rectified wire networks with four hidden layers.}
\label{fig:mnistbinacc}
\end{figure}

\begin{figure}[!t]
\begin{center}
\includegraphics[width=5.in]{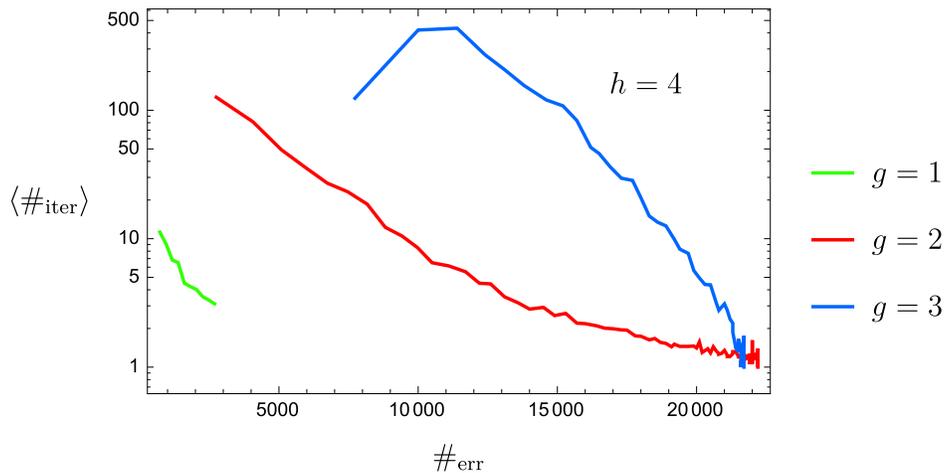}
\end{center}
\caption{SDA iterations per misclassification for the same data and networks shown in Figure \ref{fig:mnistbinacc}.}
\label{fig:mnistbiniter}
\end{figure}

\subsection{Markov sequences}\label{sec:markov}

As our final application we study a synthetic symbolic dataset where, unlike the nested majority function data, samples are generated probabilistically. The data are strings in the alphabet \{\texttt{A,B,C,D}\} generated by Markov chains. Examples of 12-symbol strings from the two classes are shown in Table \ref{tab:markovseq}. Those in class 1 are extracts from the Markov chain defined by the transition matrix
\begin{equation}\label{T}
T=\frac{1}{10}\left[
\begin{array}{cccc}
1&2&1&6\\
3&1&4&2\\
4&1&4&1\\
2&6&1&1\\
\end{array}
\right].
\end{equation}
Since $T$ is doubly stochastic, its transpose defines another Markov chain and is the source of the strings in class 2. The doubly stochastic property also ensures that the four symbols occur with equal frequency. Finally, the absence of zeros means that all bigrams have finite probability.

\begin{table}[t]
\begin{center}
\begin{tabular}{c|c}
class 1&class 2\\
\hline
\texttt{CADCCCACCADA} & \texttt{ACCBDBABDBCC}\\
\texttt{ADABDCCCDACB} & \texttt{ABDAABACCDDB}\\
\texttt{CBCCADBBCADB} & \texttt{BDBDADACDACD}\\
\texttt{DADBDADDBDCA} & \texttt{DBADCCADADDA}\\
\texttt{BCAADBBABCDB} & \texttt{BDACCCDABBDB}\\
\texttt{ABDDBADCABBD} & \texttt{CCCBDABCADBD}\\
\end{tabular}
\end{center}
\caption{Examples of 12-symbol strings extracted from the Markov chain with transition matrix \eqref{T}, in class 1, and its transpose, in class 2.}
\label{tab:markovseq}
\end{table}

Even a machine that managed to reconstruct $T$ from the data could never have perfect classification accuracy, because every 12-symbol string occurs with finite probability in either chain. The optimal classifier selects the chain (class) that gives each string the highest probability. With this criterion applied to our $T$, the true-positive rate for 12-symbol strings is 95.1\%. We also trained networks on 25-symbol strings, for which the true-positive rate is 99.1\%. These numbers set the maximum achievable classification accuracy. 

\begin{figure}[!t]
\begin{center}
\includegraphics[width=5.in]{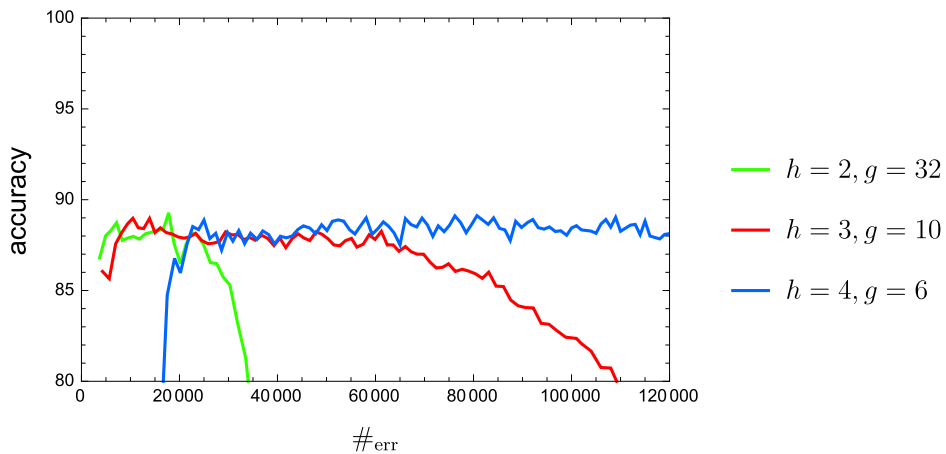}
\end{center}
\caption{Classification accuracy of 12-symbol strings generated by the Markov chain \eqref{T} or its transpose.}
\label{fig:markov12acc}
\end{figure}

\begin{figure}[!t]
\begin{center}
\includegraphics[width=5.in]{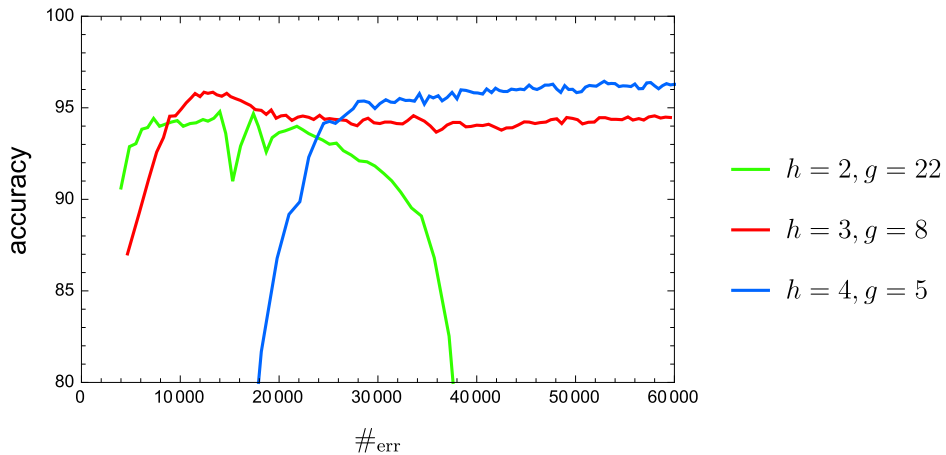}
\end{center}
\caption{Classification accuracy of 25-symbol strings generated by the Markov chain \eqref{T} or its transpose.}
\label{fig:markov25acc}
\end{figure}

In Figures \ref{fig:markov12acc} and \ref{fig:markov25acc} we compare classification accuracies for 12- and 25-symbol strings as learned on networks having approximately $2\times 10^5$ edges. We see that here the depth of the network has very little effect on the maximum accuracy. By the probabilistic nature of the data, these accuracies must eventually decline as strings having higher probability, in the class with the other label, are encountered. Deeper networks are less vulnerable to this phenomenon. The highest accuracy achieved for 12-symbol strings is 6\% below that of the optimal classifier; for 25-symbol strings the shortfall is 3\%.

We did not explore the effect of decreasing the $q$ hyperparameter for this data set.

\section{Discussion}\label{sec:open}

The rectified wire model is a significant departure from the ``standard model" of neural networks. Even so, the training algorithm we derived for this model, SDA, deserves comparison with the leading algorithm for the standard model: stochastic gradient descent (SGD) \citep{rumelhart1986learning}. In both algorithms the network parameters are updated after only a small amount of data has been seen: a single item in the case of SDA, a mini-batch for SGD. The computation of updates involves propagating data vectors in a forward pass, and the class of the data --- or the associated loss --- in a backward pass. SDA makes an extra forward pass that determines the size of each update, signaled by a deactivation event. In SGD the size of the update is set by the learning rate, a hyperparameter, and the update itself is an average over the mini-batch, whose size is another hyperparameter.

The characteristic of SDA that most differentiates itself from SGD is the monotone evolution of the network parameters. Ostensibly this is problematic: there is no ``going back" on any updates, including those made early when very little of the data had been seen. This tacitly assumes the learned state of the network is unique, or nearly so. But an alternative hypothesis is also worth considering: a high multiplicity of learned states, all equally viable. Our experiments with SDA support this hypothesis. Very different final bias settings, all giving perfect classification, were obtained even for simple classes just by changing the order of the data. Our experiments with Markov-chain generated strings showed that SDA is also not tripped up by outlier data. Although the final network in this probabilistic setting must have 100\% classification ambiguity, the accuracy in practice is near that of the optimal classifier long before this asymptotic state is reached.

Not all our experiments with SDA were successful. On the MNIST dataset the algorithm suffered from overfitting: whereas the network had learned to perfectly classify the training data, the classification accuracy on the test data was only 95\%. Does this reflect a fundamental deficiency in the rectified wire model, and the conservative learning principle more generally, or is it a result of poor execution?

An even more serious problem with the rectified wire model was exposed, in experiments, upon decreasing the weight-multiplier hyperparameter $q$. Recall that the primary motivation for the model was tractable training in the presence of network depth. Significant simplifications of the standard neural network model, such as fixing the weights and replacing loss by a constraint, were key in facilitating tractable training. Even so, the expressivity provided by the remaining bias parameters remains high (arbitrary Boolean functions), and the training algorithm (SDA) fully utilizes activations (ReLU nonlinearities) in every layer of the network. The $q$ parameter was introduced to explore behavior with respect to depth: small/large $q$ favors bias changes near the output/input nodes. Though we studied a very limited class of networks --- sparse expander nets --- the experimental results pointed uniformly in favor of $q\to 0$. In this limit the rectified wire model reduces to a much simpler model \eqref{q0limit}, comprising a single layer of rectifier gates. While this limiting model is still non-trivial and interesting, it definitely lacks depth.

Our experience with the rectified wire model highlights the difference between the capacity for depth and the learnability of depth. Although there exist bias settings in the model that efficiently represent complex Boolean functions and utilize activations on all levels, experiments show that these are not the kinds of bias settings that are learned by the tractable, conservative route. At least that is the conclusion we should draw from the experiments performed so far.

There are still some directions to explore before we can conclude that conservatively trained rectified wire networks lack the capacity to learn ``in depth". The most obvious is to try architectures very different from the sparse expander networks considered so far. These need not even be layered, although they must break symmetry if we continue to initialize all the bias parameters at zero. Specialized architectures may be inherently better at exploiting depth, as in the case of standard model networks with convolutional layers. Next, we recall that the bias updates given by the SDA algorithm can fall short of the most conservative possible. While we do not know whether the quadratic program \eqref{MP} can be solved by something as efficient as the SDA algorithm, experiments with general-purpose solvers could study the effect of update quality on training. By the same route one could study the effect of learning, in aggregate, entire mini-batches (section \ref{sec:class}, conclusions).
Another avenue is to define ``conservative" by the 1-norm on the bias updates. Our choice of the 2-norm was motivated by the uniqueness of the optimal updates. However, it seems likely that in the mini-batch setting even the 1-norm will have unique optima. Finally, it might be interesting to allow the weights in the network to vary slowly over the course of training. This could be as simple as favoring different depths over time using the $q$ parameter, or a more sophisticated scheme based on running averages.


\acks{We thank Alex Alemi, Sergei Burkov, Pat Coady, Cris Moore and Mehul Tikekar for discussions over the course of this project. We also thank the three reviewers whose comments brought significant improvements to the paper.}



\vskip 0.2in
\bibliography{18-281}

\end{document}